\documentclass[pmlr,notheorems]{jmlr}% new name PMLR (Proceedings of Machine Learning)

\usepackage{tablefootnote}
\usepackage{makecell}
\usepackage{multirow}
\usepackage{caption}
\usepackage[export]{adjustbox}

\usepackage[section]{algorithm}
\usepackage{algorithmicx,algpseudocode}
\makeatletter
\let\OldStatex\Statex
\renewcommand{\Statex}[1][3]{%
  \setlength\@tempdima{\algorithmicindent}%
  \OldStatex\hskip\dimexpr#1\@tempdima\relax}
\makeatother

\usepackage[shortlabels]{enumitem}
\setlist[itemize]{itemsep=0pt, parsep=0.5em}

\usepackage{soul}

\newif\iflong
\longfalse % switch to false in short version

\usepackage{hyperref}
\hypersetup{
    colorlinks=true,
    citecolor=NavyBlue,
    linkcolor=NavyBlue,
    urlcolor=NavyBlue,
}
\usepackage[open,openlevel=1]{bookmark}

\usepackage[nameinlink,capitalise,noabbrev]{cleveref}
\crefname{equation}{Eqn.}{Eqn.}
\crefname{algorithm}{Alg.}{Alg.}

\usepackage{longtable}% for long tables

\usepackage{booktabs}
\usepackage[load-configurations=version-1]{siunitx} % newer version

\makeatletter
\def\set@curr@file#1{\def\@curr@file{#1}} %temp workaround for 2019 latex release
\makeatother

\usepackage{amsfonts}
\usepackage{amssymb}
\usepackage{amsmath}
\usepackage{bbm}
\usepackage{mathtools}
\DeclareMathOperator*{\argmin}{arg\,min}
\DeclareMathOperator*{\argmax}{arg\,max}
\newcommand{\Dtrain}{\mathcal{D}_{\mathrm{train}}}
\newcommand{\Dval}{\mathcal{D}_{\mathrm{val}}}
\newcommand{\Dobs}{\mathcal{D}_{\textsc{obs}}}
\newcommand{\WIS}{\scriptscriptstyle \mathrm{WIS}}
\newcommand{\AM}{\scriptscriptstyle \mathrm{AM}}
\newcommand{\FQE}{\scriptscriptstyle \mathrm{FQE}}
\newcommand{\WDR}{\scriptscriptstyle \mathrm{WDR}}
\newcommand{\FQI}{\scriptscriptstyle \mathrm{FQI}}

\usepackage{amsthm}
\newtheorem{theorem}{Theorem}

\theoremstyle{definition}

\theoremstyle{remark}
\newtheorem*{remark}{Remark}
\let\oldproof\proof
\def\proof{\oldproof\unskip}

\jmlrvolume{149}
\jmlryear{2021}
\jmlrworkshop{Machine Learning for Healthcare}

\title[Model Selection for Offline RL: Practical Considerations for Healthcare]{Model Selection for Offline Reinforcement Learning: Practical Considerations for Healthcare Settings}

\author{%
\Name{Shengpu Tang}
\Email{\href{mailto:tangsp@umich.edu}{\color{black}{tangsp@umich.edu}}}\\ 
\addr Department of Electrical Engineering and Computer Science\\
University of Michigan, Ann Arbor, MI, USA
\AND
\Name{Jenna Wiens}
\Email{\href{mailto:wiensj@umich.edu}{\color{black}{wiensj@umich.edu}}}\\ 
\addr Department of Electrical Engineering and Computer Science\\
University of Michigan, Ann Arbor, MI, USA
}

\begin{document}

\maketitle

\begin{abstract}
  Reinforcement learning (RL) can be used to learn treatment policies and aid decision making in healthcare. However, given the need for generalization over complex state/action spaces, the incorporation of function approximators (e.g., deep neural networks) requires model selection to reduce overfitting and improve policy performance at deployment. Yet a standard validation pipeline for model selection requires running a learned policy in the actual environment, which is often infeasible in a healthcare setting. In this work, we investigate a model selection pipeline for offline RL that relies on off-policy evaluation (OPE) as a proxy for validation performance. We present an in-depth analysis of popular OPE methods, highlighting the additional hyperparameters and computational requirements (fitting/inference of auxiliary models) when used to rank a set of candidate policies. We compare the utility of different OPE methods as part of the model selection pipeline in the context of learning to treat patients with sepsis. Among all the OPE methods we considered, fitted Q evaluation (FQE) consistently leads to the best validation ranking, but at a high computational cost. To balance this trade-off between accuracy of ranking and computational efficiency, we propose a simple two-stage approach to accelerate model selection by avoiding potentially unnecessary computation. Our work serves as a practical guide for offline RL model selection and can help RL practitioners select policies using real-world datasets. To facilitate reproducibility and future extensions, the code accompanying this paper is available online.\footnote{\url{https://github.com/MLD3/OfflineRL_ModelSelection}}
\end{abstract}

\section{Introduction}

In recent years, researchers have begun to explore the potential of reinforcement learning (RL) to aid in sequential decision making in the context of clinical care \citep{yu2019rl4health}. For example, \citet{prasad2017mechvent} considered the problem of weaning from mechanical ventilation, \citet{komorowski2018AI_Clinician} used RL to learn optimal treatments of sepsis, and \citet{futoma2020popcorn,futoma2020soda} explored hypotension management in the ICU. This is an important step towards realizing the promise of RL beyond game-like domains. However, many challenges remain. To date, the most successful applications of RL have been in settings that assume an online paradigm in which the agent can interact with the environment \citep{mnih2015atari,silver2017go,li2010recommend,kalashnikov2018robot}, an assumption that often does not hold in other contexts. In healthcare settings, one typically only has access to historical data that were collected ``offline'', thus necessitating offline RL techniques \citep{levine2020offlineRL}. In addition, the state space associated with clinical decision making is often high-dimensional and/or continuous. Capturing such complex relationships may require learning parameterized functions (e.g., neural networks) \citep{mnih2015atari}. Thus, model selection, i.e., selecting the appropriate hypothesis class and hyperparameters of the function approximators for learning policies, is crucial in offline RL to prevent overfitting. Despite the importance of model selection, we currently lack a well-accepted training-validation framework for offline RL.

One cannot simply rely on the training-validation framework used in supervised learning, in which the final model is selected based on validation performance. In RL, measuring the equivalent of ``validation performance'' requires running the learned policy in the actual environment. While this is common practice when accurate simulators are available \citep{fujimoto2019BCQd,fujimoto2019BCQ,kumar2019BEAR,laroche2019SPIBB,liu2020MBS}, it is not always feasible in healthcare settings; it may be costly or even dangerous to try out multiple untested treatment policies on real patients in order to find out which one is the best. In healthcare, researchers must rely on off-policy evaluation (OPE) techniques for model selection in offline RL. However, this approach has its own challenges. Since many OPE estimators have been developed and they can only approximate validation performance, it is unclear how useful different OPE methods are for model selection. \citet{paine2020hyperparameter} and \citet{fu2021benchmarks} found FQE to be a useful strategy for hyperparameter selection in offline RL on several benchmark control tasks, yet the practical limitations and tradeoffs of these OPE methods have not been systematically examined. Furthermore, OPE methods often require additional model selection, i.e., they involve \textit{additional} hyperparameters that must be set \textit{a priori}. In addition, nearly every OPE method requires auxiliary models that must also be learned from data, leading to (sometimes substantial) increases in computational costs. The lack of guidelines for offline RL model selection has led the majority of empirical RL analyses on healthcare data to rely on pre-specified hyperparameters during training \citep{gottesman2018evaluating,tang2020clinician}, but this does not necessarily result in learning the best policy nor does it enable fair comparisons between different RL algorithms. 

In this work, we propose a training-validation framework that relies on OPE for model selection in offline RL settings. We investigate four popular OPE methods, highlighting practical considerations pertaining to their dependence on additional hyperparameters, auxiliary models, and their associated computational costs when used to rank a set of candidate policies. To quantify the effectiveness of each OPE method for model selection, we perform empirical experiments using a sepsis treatment task \citep{oberst2019gumbel,futoma2020popcorn}. We consider common model selection scenarios including early stopping and/or choosing neural network architectures, for both discrete state spaces (tabular) and continuous state spaces (function approximation). Given the relative computational costs and ranking quality of different OPE estimators, we propose a simple two-stage approach that combines two OPE estimators to avoid expensive computation on low-quality policies. First, we use a more efficient but less accurate OPE method to identify a promising subset of candidate policies, and second, we use a more accurate but less efficient OPE method to select a final policy from the pruned subset.

\subsection*{Generalizable Insights about Machine Learning in the Context of Healthcare}
We tackle the problem of model selection for offline RL applied to observational datasets, such as those in healthcare applications. Selecting appropriate hyperparameters for model-free value-based RL is crucial to prevent overfitting. Our contributions can be summarized as follows:
\begin{itemize}%[noitemsep]
    \item We analyze four popular OPE methods in the context of estimating validation performance of learned policies, highlighting their dependence on additional hyperparameters, auxiliary models, and their computational requirements.
    \item On the problem of using RL to learn policies for treating patients with sepsis, we empirically evaluate the effectiveness of using OPE validation performance for model selection, addressing practical considerations. 
    \item We propose a simple two-stage selection approach that effectively combines OPE estimators, balancing computational efficiency with policy performance. % and outperforms hybrid doubly robust methods;
\end{itemize}
% \vspace{-0.3em}%
RL shows promise in discovering better treatment policies from historical data, but the challenges of learning and evaluating policies in offline settings have impeded its wider adoption in high-stakes clinical domains. Our work provides one of the first empirical demonstrations of offline RL model selection using OPE methods with a focus on practical considerations; it sheds light on the pros and cons of various OPE methods for this purpose and can inform future research and enable fairer comparisons when applying offline RL to problems in healthcare. Though inspired by recent applications of RL to healthcare, many of the insights here generalize beyond healthcare decision making to other high-stakes domains where accurate simulators are unavailable, e.g., intelligent tutoring systems, dialogue systems, and online advertisement delivery. 

\section{Offline Reinforcement Learning - Problem Setup} \label{sec:background}

We consider Markov decision processes (MDPs) defined by a tuple $\mathcal{M} = (\mathcal{S}, \mathcal{A}, p, r, \mu_0, \gamma)$, where $\mathcal{S}$ and $\mathcal{A}$ represent the state and action spaces, $p(s'|s,a)$ and $r(s,a)$ specify the transition and reward models, $\mu_0(s)$ is the initial state distribution, and $\gamma \in [0,1]$ is the discount factor. In this work, we consider settings with discrete/continuous state spaces and a discrete action space. A probabilistic policy $\pi(a|s)$ specifies a mapping from each state to a probability distribution over actions. When the policy is deterministic, $\pi(s)$ refers to the action with $\pi(a|s)=1$. In RL, we aim to find an optimal policy \( \pi^* = \argmax_{\pi} J(\pi; \mathcal{M}) \) that has the maximum expected performance in $\mathcal{M}$. The value of policy $\pi$ with respect to $\mathcal{M}$ is defined as \( v(\pi) = J(\pi; \mathcal{M}) = \mathbb{E}_{s \sim \mu_0}[V^{\pi}(s)] \); here, the state-value function is defined as \( V^{\pi}(s) = \mathbb{E}_{\pi}\mathbb{E}_{\mathcal{M}} \left[ \sum_{t=1}^{\infty}\gamma^{t-1} r_t \ | \  s_1 = s \right] \). The action-value function, $Q^{\pi}(s,a)$, is defined by further restricting the action taken from the starting state. 

Throughout this paper, we focus on the \textit{offline} setting. In contrast to an \textit{online} setting in which one can interact with $\mathcal{M}$ (either directly or via a simulator), in offline RL, one only has access to data that were previously collected. More precisely, one has access to an observational dataset ${\Dobs=\{ s_i, a_i, r_i, s'_i \}_{i=1}^{N}}$ consisting of $N$ transitions, collected from $\mathcal{M}$ by one or multiple behavior policies denoted by $\pi_b$. The dataset can also be considered as a collection of $m$ episodes, $\Dobs=\{ \tau_j \}_{j=1}^{m}$, where each episode is $\tau = (s_1, a_1, r_1, \cdots, s_{{\scriptscriptstyle L}}, a_{{\scriptscriptstyle L}}, r_{{\scriptscriptstyle L}})$ with $L$ denoting the length of episode $\tau$.  Given the challenges such a setting presents (i.e., limited exploration), we cannot necessarily identify an optimal policy \citep{lange2012batch,levine2020offlineRL}; instead, we aim to learn a policy with low suboptimality $J(\pi^*; \mathcal{M}) - J(\pi; \mathcal{M})$. We focus on model-free, value-based methods for offline policy learning -- specifically, fitted Q iteration (FQI) (\citealp{ernst2005batchRL,riedmiller2005NFQ}; see \cref{alg:FQI} in \cref{appx:algs}) -- as it applies to problems with either discrete state spaces (the tabular setting) or continuous state spaces (the function approximation setting, where states are represented as feature vectors $\mathbf{x}(s) \in \mathbb{R}^d$). 

Given a policy $\pi$, its performance (i.e., $J(\pi; \mathcal{M})$, the expected cumulative reward in $\mathcal{M}$) can be estimated from historical data collected by some other policy (or policies) $\pi_b$ using off-policy evaluation (OPE) \citep{voloshin2019OPE}. OPE is inherently difficult because it requires estimating counterfactuals: we want to understand what \textit{would have} happened had the agent acted differently from how it acted in historical data. We discuss four popular OPE methods in \cref{sec:applyOPE} and show how they can be used to estimate validation performance in the context of model selection for offline RL.

\section{Offline RL Model Selection using OPE}

We present a model selection framework for offline RL based on OPE. First, we analyze four common OPE methods in the context of model selection. We highlight practical considerations including their dependence on additional hyperparameters and auxiliary models, as well as computational requirements. Second, we propose a simple two-stage approach for model selection that combines two OPE estimators in order to balance the trade-off between ranking quality and computational efficiency.
\label{sec:pipeline}
\begin{figure}[h]
    \centering
    \includegraphics[scale=1,trim={3cm -0.5cm 4cm -0.5cm}]{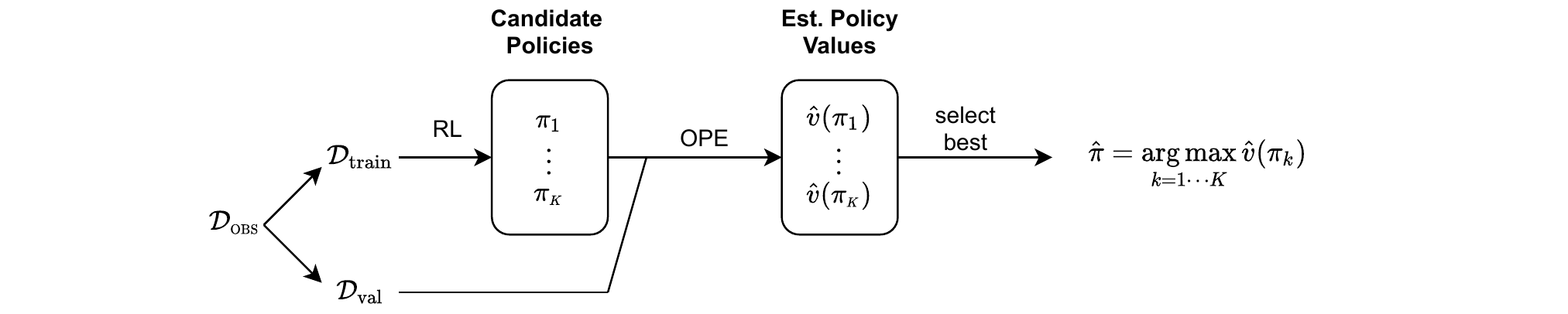}
    \caption{The observational dataset $\Dobs$ is first split into two mutually exclusive parts, $\Dtrain$ and $\Dval$. The training data $\Dtrain$ are used to derive a set of candidate policies by running some RL algorithm (e.g., FQI). Each policy is evaluated using an OPE method on the validation data $\Dval$ to obtain its estimated policy value $\hat{v}(\pi_k)$. Finally, we select the policy that achieves the highest estimated validation performance. }
    \label{fig:pipeline}
\end{figure}

\subsection{A Model Selection Framework for Offline RL} 
The model selection framework commonly used in supervised learning \citep{murphy2012ml,Goodfellow-et-al-2016} can be adapted for offline RL, where OPE methods are used to estimate the validation performance (\cref{fig:pipeline}). Given an observational dataset, we first split it into two mutually exclusive parts: a training set $\Dtrain$ and a validation set $\Dval$, assumed to follow the same data distribution (the same MDP and the same behavior(s)). During training, a set of $K$ candidate policies $\{\pi_k\}_{k=1}^{K}$ are derived from the $\Dtrain$ using an algorithm such as FQI (\cref{alg:FQI}), where each policy is learned with a different hyperparameter setting (e.g., training iterations, neural network architectures). Using OPE, we then assign a scalar score to each candidate policy $\pi_k$ by estimating its value using $\Dval$. The scores collectively result in a ranking that can then be used to identify the best policy. While seemingly straightforward, such a model selection framework has not been consistently applied, in part because of the practical challenges that one encounters in offline settings.

%Moreover, implementing the framework from \cref{fig:pipeline} raises several practical questions that we will discuss in the next subsection. 

%In some past work, e.g., \citet{paine2020hyperparameter}, a separate validation set is not created; instead, the same data is used for both training and validation. While this increases the size of the training set and could potentially lead to better policies, it could increase the risk of overfitting at the validation step.

\subsection{Using OPE for Model Selection}\label{sec:applyOPE}

Compared to measuring validation performance in a supervised learning setting (by checking predictions against ground-truth labels), measuring validation performance in offline RL cannot be simply done using the observed returns from the validation set. This is because from the same state, the evaluation policy may choose a different action (or a different distribution over actions) compared to the behavior policy. Thus, we have to use OPE approaches, which themselves rely on additional hyperparameters or learning additional models. In this section, we provide a qualitative comparison of popular OPE methods. We describe each OPE and its mathematical formalization. In addition, we describe any additional hyperparameters that must be set \textit{a priori}, dependencies on auxiliary models, and computational requirements when used for model selection. As shown in \cref{fig:pipeline}, we use the validation data $\Dval$ as input to OPE, but to simplify the exposition below, we use $\mathcal{D}$ for notational convenience. We assume $\mathcal{D}$ consists of $m$ episodes and $N$ transitions, and there are $K$ candidate policies. For computational complexity, we focus on the function approximation setting that involves fitting classification/regression models; the equivalent of these model dependencies in the tabular setting can be obtained efficiently via maximum likelihood estimation (see \cref{appx:OPE-compute}). 

\subsubsection{Weighted Importance Sampling (WIS)}

Due to differences between the evaluation policy and the behavior policy, each episode can be more (or less) likely to occur under the evaluation policy compared to the behavior policy. Importance sampling estimates the value of the evaluation policy from the validation data by re-weighting episodes according to how relatively likely they are to occur \citep{puaduraru2013OPE,voloshin2019OPE}. 

\medskip\noindent \textit{Formalization.} Given a policy $\pi$ that we aim to evaluate and a behavior policy $\pi_b$ that generated episode $\tau = (s_1, a_1, r_1, \cdots, s_{{\scriptscriptstyle L}}, a_{{\scriptscriptstyle L}}, r_{{\scriptscriptstyle L}})$, we define the per-step importance ratio $\rho_t = \frac{\pi(a_t|s_t)}{\pi_b(a_t|s_t)}$ and the cumulative importance ratio $\rho_{1:t} = \prod_{t'=1}^{t}\rho_{t'}$. The WIS estimator is calculated as 

\[ \hat{v}_{\WIS}(\pi) = \frac{1}{m} \sum_{j=1}^{m} G_{\WIS}(\tau_j; \pi), \text{ where } G_{\WIS}(\tau_j; \pi) = \frac{\rho_{1:L^{(j)}}^{(j)}}{w_{L^{(j)}}} \left( \sum_{t=1}^{L^{(j)}} \gamma^{t-1}r_t^{(j)} \right) , \]

\medskip\noindent where $w_t = \frac{1}{m} \sum_{j=1}^{m}\rho^{(j)}_{1:t}$ is the average cumulative importance ratio at horizon $t$. %WIS can be viewed as a weighted average of the Monte Carlo returns in $\mathcal{D}$ according to how likely each sample episode were to occur under $\pi$ relative to $\pi_b$. 
We use WIS (a biased but consistent estimator) rather than ordinary importance sampling (an unbiased estimator) because it can provide more stable performance \citep{puaduraru2013OPE,voloshin2019OPE}. 

\medskip\noindent \textit{Hyperparameters.} When the evaluation policy $\pi$ is deterministic such that $\pi(a|s)$ is exactly $1$ for a single action and $0$ for all others, an episode will receive zero weight if $\pi$ and $\pi_b$ take different actions at \textit{any} timestep. Such episodes are thus effectively discarded, and only episodes with non-zero weights are included in WIS calculation \citep{gottesman2018evaluating}. This can be especially problematic with small sample sizes and long episodes, or when $\pi$ and $\pi_b$ differ significantly. To overcome this issue, it is common to \textit{soften} policy $\pi$ with the hope of increasing the effective sample size and reducing variance \citep{raghu2017deepRL,raghu2017continuous,komorowski2018AI_Clinician}: e.g., \( \tilde{\pi}(a|s) = \mathbbm{1}_{a=\pi(s)} (1-\varepsilon) + \mathbbm{1}_{a\neq\pi(s)} (\frac{\varepsilon}{|\mathcal{A}|-1}) \). While this results in a value estimate for $\tilde{\pi}$ rather than for $\pi$, it is usually sufficiently close. The hyperparameter $\varepsilon$ is typically set to a small value using rules of thumb (e.g., $\varepsilon=0.01$). 

\medskip\noindent \textit{Dependencies on Auxiliary Models.} Since we typically do not know the true behavior policy, we need to estimate it from data \citep{hanna2019importance,chen2019multiple}. This can be obtained as either $\hat{\pi}_b(a|s) = \frac{\operatorname{count}(s,a)}{\operatorname{count}(s)}$ for discrete states, or for continuous states, by fitting a multiclass classification model that predicts the action distribution of the behavior policy given a state, trained using the softmax cross-entropy loss \citep{torabi2018cloning}. 

%A common assumption is that the validation data were generated using a single behavior policy \citep{raghu2017continuous,komorowski2018AI_Clinician} even though it may be hard to verify in real problems (e.g., when historical data are generated by different clinicians).
%\footnote{In experiments we will explore settings when this assumption is violated, i.e., when data follows a mixture of multiple behavior policies but our $\hat{\pi}_b$ estimation assumes a single behavior policy. }

\medskip\noindent \textit{Computational Cost.} WIS requires fitting a single model $\hat{\pi}_b$ over all observed state-action pairs $\{\mathbf{x}(s_i), a_i\}_{i=1}^{N}$ and performing one inference run on all observed states $\{\mathbf{x}(s_i)\}_{i=1}^{N}$, resulting in $O(N)$ complexity for both fitting and inference. Importantly, the computational cost can be considered amortized since it does not depend on $K$, the number of candidate policies.\footnote{To be precise, applying WIS on each policy will require new importance ratios to be computed, but the cost of such arithmetic operations is minimal compared to performing fitting/inference of a parameterized model (e.g., neural networks).} This allows us to apply WIS to a large candidate set with relatively little increase in computation.

\subsubsection{Approximate Model (AM)}

In the online setting, we evaluate a policy by running it in the actual environment. While we do not have access to the true environment in the offline case, we can build an approximate model of the environment from $\mathcal{D}$.  AM is a model-based approach that estimates the policy value by first learning a model of the environment and then performing policy evaluation using the learned model \citep{jiang2016doubly,voloshin2019OPE}. As the name suggests, AM is only an approximation but may be sufficient for providing a reasonable estimate of the policy value. 

\medskip\noindent \textit{Formalization.} To learn a set of models for AM, we need to fit the transition dynamics $\hat{p}(s'|s,a)$ and the reward function $\hat{r}(s,a)$, and if necessary, the initial state distribution $\hat{\mu}_0(s)$, and/or the termination model $\hat{p}_{\mathrm{term}}(s,a)$ for episodic problems. Given these estimated parameters of the environment, we can estimate the value functions $\widehat{V}_{\AM}$ and/or $\widehat{Q}_{\AM}$ by performing analytic or iterative policy evaluation for the given policy (only applicable to the tabular setting, i.e., discrete state and action spaces), or use the learned model as a simulator to obtain average returns over Monte-Carlo rollouts \citep{jiang2016doubly,voloshin2019OPE}. The final policy value can be calculated as the average value of initial states (or as the expectation over the estimated initial state distribution):
\begin{align*}
    \hat{v}_{\AM}(\pi) &= \frac{1}{m} \sum_{j=1}^{m} \widehat{V}^{\pi}_{\AM}(s_1^{(j)}) \textrm{ or } \sum_{s_1\in\mathcal{S}} \hat{\mu}_0(s_1) \widehat{V}^{\pi}_{\AM}(s_1), \\
    \text{ where } \widehat{V}^{\pi}_{\AM}(s) &= \hspace{1em}\underset{\mathclap{\substack{a_t \sim \pi(\cdot|s_t) \\ r_t \sim \hat{r}(s_t,a_t) \\ s_{t+1} \sim \hat{p}(\cdot|s_t,a_t)}}}{\mathbb{E}} \hspace{2em} \left[ \sum_{t=1}^{H} \gamma^{t-1}r_t \ \Big| \ s_1 = s \right] .
\end{align*}

\noindent \textit{Hyperparameters.} When Monte-Carlo rollouts are used for evaluation, in lieu of a termination model (which itself requires specifying a threshold on the predicted termination probability), it is more common to directly set the maximum length of rollout $H$, which corresponds to the evaluation horizon. One possible heuristic is to set $H$ based on the typical episode length $\tilde{L}$ in historical data. However, if we expect the learned policy to outperform the behavior policy in a particular way and produce longer or shorter episodes (e.g., keeping patients healthy for a longer time, or discharge patients alive earlier), it may be more appropriate to adjust $H$ depending on the specific setting. When analytic or iterative policy evaluation is used (for tabular settings), we consider $H\to\infty$. 

\medskip\noindent \textit{Dependencies on Auxiliary Models.} As mentioned above, AM depends on the following auxiliary models: $\hat{p}$, $\hat{r}$, and possibly $\hat{\mu}_0$, $\hat{p}_{\mathrm{term}}$. 
\begin{itemize}[noitemsep]
    \item For discrete state spaces, these quantities can be found using maximum likelihood estimation: $\hat{p}(s'|s,a) = \frac{\operatorname{count}(s,a,s')}{\operatorname{count}(s,a)}$, $\hat{r}(s,a) = \frac{\sum_{i=1}^{N}r_i\mathbbm{1}_{(s_i=s, a_i=a)}}{\operatorname{count}(s,a)}$, $\hat{\mu}_0(s) = \frac{\operatorname{count}(s_1=s)}{m}$ (the termination model can be considered part of the transition model).
    \item For continuous state spaces, these models can be found using supervised learning. Even though the true environment may be stochastic and one could model the full probability distributions of transitions and rewards in low-dimensional settings (e.g., Gaussian dynamics network \citep{torabi2018cloning,yu2020MOPO,kidambi2020MOReL}), it makes the Monte-Carlo rollouts more susceptible to noise. One could generate multiple rollouts from each state, but this requires a new hyperparameter and significantly more computation. Alternatively, it is common to impose simplifying assumptions and instead learn the following regression models \citep{voloshin2019OPE,rajeswaran2020game}: $\hat{\Delta}_s(\mathbf{x}(s),a)$ an autoregressive model to predict the expected change in state features and $\hat{r}(\mathbf{x}(s),a)$ to predict the expected immediate reward. These models can be trained with mean squared error as the loss function. Furthermore, instead of modeling the full probability distribution of initial states as $\hat{\mu}_0$, we can directly use the observed initial states as the starting states for rollouts. Termination models are not typical, and instead we use $H$ as the fixed rollout length. 
\end{itemize}

\medskip\noindent \textit{Computational Cost.} Similar to WIS, the cost of auxiliary model fitting does not depend on $K$, the number of candidate policies; only a single set of models needs to be learned using the following pattern sets: $\{\langle\mathbf{x}(s_i), a_i\rangle, \mathbf{x}(s'_i)-\mathbf{x}(s_i)\}_{i=1}^{N}$ for $\hat{\Delta}_s(s,a)$, and $\{\langle\mathbf{x}(s_i), a_i\rangle, r_i\}_{i=1}^{N}$ for $\hat{r}(s,a)$. The fitting complexity is thus $O(N)$. However, unlike WIS, the inference cost for AM scales linearly with $K$, because each candidate policy requires a different set of rollouts. The overall inference cost is $O(HmK)$, for advancing $H$ steps from the $m$ initial states following each of the $K$ candidate policies, assuming they are all deterministic. The inference cost for a set of probabilistic policies is approximately $O(|\mathcal{A}|^H HmK)$ in the worst case, because we need to consider all possible action sequences with nonzero probability following each initial state and weight them appropriately.

\subsubsection{Fitted Q Evaluation (FQE)}

Instead of reweighting the observed episodes (WIS) or building an approximate environment model (AM), we can algorithmically summarize the performance of the evaluation policy using the Q-function, a mapping from state-action pairs to the expected cumulative reward. This can be done by running FQE, an off-policy value-based algorithm that directly learns $\widehat{Q}^{\pi}_{\FQE}(s,a)$, the Q-function of $\pi$, from historical data. 

\medskip\noindent \textit{Formalization.} FQE (\cref{alg:FQE}) is a value-based, temporal difference algorithm \citep{le2019FQE} and is closely related to expected SARSA with respect to a fixed policy for policy evaluation \citep{van2009expSARSA}. Each iteration of the algorithm applies the Bellman equation for $Q^{\pi}$ to compute the bootsrapping targets for all transitions from $\mathcal{D}$ using the previous estimates (line 4) and then applies function approximation via supervised learning (lines 5-6). The algorithm outputs $\widehat{Q}^{\pi}_{\FQE}(s,a)$, the estimated Q-function of $\pi$. The final value estimate can be calculated as the average value for the initial states in $\mathcal{D}$: 
\[ \hat{v}_{\FQE}(\pi) = \frac{1}{m} \sum_{j=1}^{m} \widehat{V}^{\pi}_{\FQE}(s_1^{(j)}), \text{ where } \widehat{V}^{\pi}_{\FQE}(s) = \sum_{a\in\mathcal{A}} \pi(a|s) \widehat{Q}^{\pi}_{\FQE} (s, a). \]

\noindent \textit{Hyperparameters.} Similar to AM, FQE also has a hyperparameter $H$, the number of FQE iterations, which corresponds to the evaluation horizon (the output $\tilde{Q}_{H}$ is the $H$-step Q-function of $\pi$). 

\medskip\noindent \textit{Dependencies on Auxiliary Models.} For each candidate policy $\pi_k$, we need to fit a sequence of Q-networks $\{\tilde{Q}_h\}_{h=1}^{H}$ (\cref{alg:FQE}, line 6). 

\medskip\noindent \textit{Computational Cost.} Unlike WIS and AM where the number of auxiliary models does not depend on the number of candidate policies (the fitting costs are amortized), FQE scales linearly with $K$ because a different set of models is needed for each $\pi_k$. Furthermore, due to bootstrapping where the target value of the current iteration depends on the result of the previous iteration (\cref{alg:FQE}, line 4), each FQE iteration involves both inference and fitting of $\tilde{Q}$, resulting in a total fitting complexity of $O(HNK)$ for both deterministic and probabilistic candidate policies. Notably, the process of applying FQE to each policy cannot be parallelized due to this dependency. However, multiple policies can be evaluated in parallel. To obtain the final value estimate, each $\tilde{Q}_H$ network is applied to the initial states from $m$ episodes, resulting in inference complexity of $O(mK)$. Overall, FQE is much more expensive compared to WIS/AM in terms of computational cost (and more expensive than policy learning) and can take much longer to run.

\subsubsection{Weighted Doubly Robust (WDR) Estimator} \label{sec:WDR}

WDR is a hybrid method that combines importance sampling from WIS with value estimators $\widehat{V}^{\pi}$ and $\widehat{Q}^{\pi}$ obtained from AM or FQE \citep{jiang2016doubly,thomas2016MAGIC}. In addition to reweighting the observed returns (as in WIS), WDR leverages the value estimates by incorporating them at every time step with the hope of reducing the overall variance. 

\medskip\noindent \textit{Formalization.} We consider two versions of this estimator: WDR-AM and WDR-FQE. The estimator is defined as follows (with $\rho_{1:0} = 1$): 
\begin{align*}
    \hat{v}_{\WDR}(\pi) &= \frac{1}{m} \sum_{j=1}^{m} G_{\WDR}(\tau_j; \pi), \\
    \text{ where } G_{\WDR}(\tau_j; \pi) &= \sum_{t=1}^{L^{(j)}} \left[ \frac{\rho_{1:t}^{(j)}}{w_{t}} \gamma^{t-1} r_t^{(j)} - \left( \frac{\rho_{1:t}^{(j)}}{w_{t}} \widehat{Q}^{\pi}(s_t^{(j)}, a_t^{(j)}) - \frac{\rho_{1:t-1}^{(j)}}{w_{t-1}} \widehat{V}^{\pi}(s_t^{(j)}) \right) \right] . 
\end{align*}
WDR is often assumed to be better than WIS and FQE/AM used by themselves because in theory, it is a consistent low-variance estimator if either importance ratios (i.e., behavior policy) or the value function estimates are properly specified (hence the name ``doubly robust''). However, more recent empirical work has shown that WDR methods can be less accurate and have larger errors than its constituent OPE methods due to various reasons, such as stochastic transitions/rewards or high variance in importance ratios \citep{jiang2016doubly,thomas2016MAGIC,voloshin2019OPE}. 

\medskip\noindent \textit{Hyperparameters.} WDR has the hyperparameters of its constituent parts (both $\varepsilon$ and $H$). 

\medskip\noindent \textit{Dependencies on Auxiliary Models.} Beyond the models from WIS/AM/FQE, the two WDR approaches do not rely on any additional auxiliary models. 

\medskip\noindent \textit{Computational Cost.} There is no additional computation needed for model fitting. However, there are some notable differences in terms of inference complexity because WDR requires estimates of both $Q$ and $V$ for all $N$ state-action pairs, rather than the $m$ initial states. 
\begin{itemize}[noitemsep,parsep=1pt]
    \item AM estimates $Q$/$V$ by performing Monte-Carlo rollouts. Given the softened policy $\tilde{\pi}$, in order to obtain $\widehat{V}^{\tilde{\pi}}_{\AM}(s) = \sum_{a\in\mathcal{A}} \tilde{\pi}(a|s)\widehat{Q}^{\tilde{\pi}}_{\AM}(s,a)$ for each state, we need to perform rollouts that consider all possibilities of the future following $\tilde{\pi}$, which is exponential in the rollout length $H$. This results in $O(|\mathcal{A}|^{H} HNK)$ inference complexity to compute WDR-AM exactly, since the exponential rollouts need to be done for all $N$ transitions and all $K$ candidate policies, which is prohibitively expensive. If the evaluation policy is deterministic, then a slightly more practical solution is to use the softened $\tilde{\pi}$ to calculate importance ratios, but use the original $\pi$ to perform rollouts, resulting in $O(HNK)$ inference complexity, which is still relatively expensive. In light of these computation requirements, we omitted WDR-AM from our experiments on problems with continuous state spaces in favor of WDR-FQE. % If $N \approx mL$ and $H \approx L$, then WDR-AM requires $NH \approx mL^2$ which is approximately quadratic computation per episode, which quickly becomes infeasible even for episodes of reasonable lengths. % 
    \item Since FQE already learns parameterized functions for $\widehat{Q}$ (and thus for $\widehat{V}$), we only need to perform additional inference runs using all $N$ state-action pairs as input, resulting in $O(NK)$ inference complexity. 
\end{itemize}

\vspace{-0.8em}
\begin{table}[h]
    \centering
    \caption{Qualitative comparison of OPE methods when used for model selection over $K$ candidate policies, including: any additional hyperparameters that must be set \textit{a priori}, dependencies on additional estimated quantities (discrete state space) or auxiliary supervised models (continuous state space), and approximate computational complexity. For hybrid approaches (WDR), the listed complexity estimates are for additional computations not already done by the constituent OPE estimators. We assume the dataset consists of $m$ episodes and $N$ transitions, and all candidate policies are deterministic. $\varepsilon$ is the policy softening hyperparameter and $H$ is the length of the evaluation horizon. Other notation are described in \cref{sec:applyOPE}. A more detailed breakdown can be found in \cref{appx:OPE-compute}. }
    \label{tab:OPE-summary}
    \scalebox{0.81}{
    \begin{tabular}{cccccc}
    \toprule
    \multirow{2}{*}{\textbf{OPE}} & \multirow{2}{*}{\textbf{Hyperparameters}} & \multirow{2}{*}{\textbf{Dependencies}} & \multicolumn{2}{c}{\textbf{Complexity}} \\
    &&& Fitting & Inference \\
    \midrule
    WIS & $\varepsilon$ & $\hat{\pi}_b$ & $O(N)$ & $O(N)$ \\
    AM & $H$ & $\hat{p}$, $\hat{r}$ & $O(N)$ & $O(HmK)$ \\
    FQE & $H$ & $\{\tilde{Q}_h\}_{h=1}^{H}$ for each $\pi_k$ & $O(HNK)$ & $O(mK)$ \\
    WDR-AM & $\varepsilon$; $H$ & WIS + AM & WIS + AM & $O(|\mathcal{A}|^{H}HNK)$\\
    WDR-FQE & $\varepsilon$; $H$ & WIS + FQE & WIS + FQE & $O(NK)$ \\
    \bottomrule
    \end{tabular}
    }%
\end{table}

\subsection{Combining Multiple OPE Estimators for Model Selection} \label{sec:combine-OPEs}

While every OPE outputs an estimated value of the evaluation policy, each method does so in a different way, e.g., by estimating the behavior policy, estimating the environment model, or directly estimating the $Q$ function. Based on past OPE literature (that focused on accuracy in estimating policy value \citep{voloshin2019OPE,fu2021benchmarks}) and our analyses on computational costs (\cref{tab:OPE-summary}), we note that WIS/AM are more efficient but tend to be less accurate, whereas FQE is less efficient but tends to be more accurate. Given a large candidate set, it might not be worthwhile applying expensive OPE methods on low-quality policies. With this insight, we propose a simple, practical procedure for combining two OPE estimators.

\medskip\noindent \textbf{Two-stage selection.} Instead of selecting the best policy based on ranking from a single OPE method, we do this in two stages: in the first stage, a computationally efficient OPE (e.g., WIS) is applied to the entire candidate set in order to identify a promising subset (e.g., the top percentiles); in the second stage, a more accurate OPE (e.g., FQE) is used to rank within the pruned subset. The performance of the two-stage selection procedure depends on the initial subset size $\alpha$ as well as choices of the two OPE estimators. This additional hyperparameter $\alpha$ needs to be prespecified; in practice, this can be determined based on available computational resources. $\alpha$ dictates the total computational cost as well as to what extent the final selected policy depends on each of the two OPE estimators: $\alpha=1$ defaults to the first OPE estimator, while $\alpha=K$ defaults to the second OPE estimator. See \cref{appx:two-stage} for a detail discussion and theoretical analyses. 
% In this way, we can reduce the overall computational cost while potentially maintaining the performance the more expensive / more accurate OPE method. 

\medskip As baselines for comparison, we consider \textbf{average score} and \textbf{average ranking}, two naive ways of combining OPE scores that could take advantage of multiple OPE methods, but do not benefit from the computational savings as the two-stage approach and may be susceptible to outliers due to high variance of certain OPE estimators. We also compare to WDR (\cref{sec:WDR}), a hybrid method that combines WIS with AM/FQE.

%  and the potential trade-off between accuracy and computational costs
% \tablefootnote{We make the simplifying assumption that model inference has the \textit{same cost} as model fitting. In reality, for neural network function approximators, model inference only involves one forward pass per sample, whereas model fitting may involve multiple forward/backward passes per sample. }

% This results in different computational costs when applied to a set of candidate policies 
% In certain settings it might be easier to avoid misspecification for a particular method: e.g., WIS under a single behavior policy or AM under deterministic MDPs. Thus, it can be beneficial to consider multiple OPE metrics when performing evaluation or model selection \citep{gottesman2019combining}. 

% The two approaches above require us to compute all OPE validation scores for every candidate policy. 

% we outline a few different strategies to combine multiple OPE estimators . 

%  that incur computation even more expensive than the policy learning process

\section{Experimental Setup}
To test the utility of different OPE methods in model selection for offline RL, we performed a series of empirical experiments and evaluated how well OPE estimates of validation performance can help in selecting policies among a candidate set. All of our experiments rely on the following steps: (i) generating historical data $\mathcal{D}_{\mathrm{train}}$ and $\mathcal{D}_{\mathrm{val}}$ from a simulator, (ii) applying an RL algorithm on $\mathcal{D}_{\mathrm{train}}$ and obtaining the set of candidate policies by using different hyperparameters during training, (iii) using OPE and $\mathcal{D}_{\mathrm{val}}$ to estimate validation performance for each candidate policy, and (iv) quantifying the ranking quality and regret of OPE estimators with respect to true values of candidate policies. We discuss the details of these steps below. %For reproducibility, upon publication, all code will be released via a public repository.

% To simulate an offline RL problem, we structure our experiments into three steps: (i) generate historical using a simulator, (ii) apply an RL learning algorithm, (iii) compute the learning curves across FQI iterations using offline metrics. 

% FQI on $\mathcal{D}$ with some chosen function class $\mathcal{F}$

% \red{OUTDATED} We seek to empirically evaluate to what extent metrics based on offline data can guide us in reducing the negative effects of unlearning. While early stopping might seem a very natural solution, past work has usually assumed access to the environment and used an online oracle to guide early stopping \citep{fu2019bottleneck}. This work provides one of the first empirical evaluations of early stopping using offline data alone. Compared to \citet{paine2020hyperparameter} which considered the general problem of hyperparameter selection in deep offline RL, we focus on choosing the training iterations (early stopping) at a more granular level to address unlearning rather than optimizing for architecture hyperparameters. We also apply OPE methods other than FQE \citep{le2019FQE}, which was used exclusively in \citet{paine2020hyperparameter}. 

\subsection{Environment \& Dataset Generation}
While this work focuses on the offline setting, we consider a simulated problem because we need to evaluate the quality of policy ranking according to each OPE with respect to the true policy performance. We use the sepsis simulator adapted from \citet{oberst2019gumbel} and \citet{futoma2020popcorn}, which is modeled after the physiology of patients with sepsis. More details of the simulation setup can be found in \cref{appx:sepsisSim} and the source code; below we provide a brief description. The underlying patient state consists of 5 discrete variables: a binary indicator for diabetes status, and 4 ordinal-valued vital signs (heart rate, blood pressure, oxygen concentration, glucose). We consider two state representations: a discrete state space with $|\mathcal{S}| = 1,440$, and a continuous state space where each state corresponds to a feature vector $\mathbf{x}(s) \in \{0,1\}^{21}$ that uses a one-hot encoding for each underlying variable. There are 8 actions based on combinations of 3 binary treatments: antibiotics, vasopressors, and mechanical ventilation, where each treatment can affect certain vital signs and may raise/lower their values with pre-specified probabilities. A patient is discharged only when all vitals are normal and all treatments have been stopped; death occurs if 3 or more vitals are abnormal. Rewards are sparse and only assigned at the end of each episode, with $+1$ for survival and $-1$ for death, after which the system transitions into the respective absorbing state. For data generation, episodes are truncated at a maximum length of 20. We used a uniformly random policy to generate 10 pairs of datasets (for training and validation), each with large number of episodes ($m=$ 10,000) and a different random seed. While this data generation process ensures sufficient exploration in the state-action space, it makes certain assumptions that might not hold in practice; we explore some of these issues in \cref{sec:results-comparison}. 
% \footnote{\url{https://github.com/clinicalml/gumbel-max-scm/}}

\subsection{Implementation Details}
To learn the policies, we ran FQI (\cref{alg:FQI}) for up to 50 iterations using a neural network function approximator. We used 10,000 episodes following the uniformly random policy for both training and validation unless otherwise specified, and repeated each experiment for 10 replication runs. We followed the procedures outlined in \cref{sec:applyOPE} to learn appropriate auxiliary models (either via MLE or as individual supervised tasks) and applied OPE methods using the learned models. We used $\varepsilon=0.01$ and $H=20$ as the default OPE hyperparameters following recommendations provided in \cref{sec:applyOPE}. By default, we used neural networks consisting of one hidden layer with 1,000 neurons and ReLU activation to allow for function approximators with sufficient expressivity. We trained these networks using the Adam optimizer (default settings) \citep{adam} with a batch size of 64 for a maximum of 100 epochs, applying early stopping on $10\%$ ``validation data'' (specific to each supervised task) with a patience of 10 epochs. We minimized either the mean squared error (MSE) for regression tasks (each iteration of FQI/FQE, dynamics/reward models of AM), or the cross-entropy loss with softmax activation for classification tasks (behavior policy for WIS). For FQI/FQE, we also added value clipping (to be within the range of possible returns $[-1,1]$) when computing bootstrapping targets to ensure a bounded function class and encourage better convergence behavior \citep{mnih2015atari}.

% We also include a linear regressor with the same feature vectors in certain cases, but due to its limited expressivity and poor performance, we show those results in the appendix. 

\subsection{Evaluation}

Given our goal is to select the best (or a good) policy from the candidate set, the OPE validation scores does not necessarily need to match exactly with true policy performance; they only need to correlate well \citep{irpan2019OPC}. Thus, we do not present mean squared error (MSE), a common metric used to evaluate OPE methods \citep{voloshin2019OPE}. Instead, we report the following quantitative metrics \citep{paine2020hyperparameter,fu2021benchmarks}:

\begin{itemize}[noitemsep,topsep=0pt]
    \item Spearman's rank correlation $\rho$ between OPE scores and ground-truth policy values. This measures the overall ranking quality, i.e., how well an OPE method can rank policies according to their true performance from high to low. 
    \item Regret, defined as \, $\max\limits_{\mathclap{\scriptscriptstyle k=1 \cdots K}} v(\pi_k) - v(\pi_{k^*})$ where $k^* = \operatorname{arg}\max\limits_{\mathclap{\scriptscriptstyle k'=1 \cdots K}} \hat{v}(\pi_{k'})$ corresponds to the candidate policy with the best OPE validation score. 
    %$v(\pi_{k^*}) - v(\hat{\pi}^*)$, where $k^* = \argmax_{k=1 \cdots K} v(\pi_k)$ corresponds to the best policy among the candidate set. 
    This measures how far the identified ``best'' policy is from the actual best policy among the candidate set. We also report \mbox{top-$n$} regret (i.e., regret@$n$), which is the smallest regret for the $n$ policies with the highest OPE scores.%\footnote{Note that \textsf{k} is not to be confused with $k$, the index of candidate policies (stylized differently).} 
    % \item Suboptimality, defined as $v(\pi^*) - v(\hat{\pi}^*)$, measures how far the identified ``best'' policy is from the true optimal policy.
\end{itemize}
We evaluated policies analytically using the matrix inversion method and ground-truth MDP model parameters $p$, $r$, and $\mu_0$ from the sepsis simulator to obtain the true value $v(\pi)$. While regret ultimately determines how well model selection can identify the best policy, we consider rank correlation and regret@$n$ as they indicate to what extent an OPE method can be used for identifying the promising subset in the two-stage selection. 

% We want to know how well early stopping based on a metric can help find a policy with better performance. We start by plotting the learning curves and visually inspect the qualitative similarity in their shapes/trends compared to the ground truth curve. 
    % \item Mean-squared error (MSE). This is commonly used to evaluate OPE methods \citep{voloshin2019OPE} and is defined as $\frac{1}{K}\sum_{k=1}^{K}(\hat{V}(\pi_k) - J(\pi_k; \mathcal{M}))^2$. \red{no MSE in main results}
    % \item Policy performance, $J(\pi_{k^*}; \mathcal{M})$ evaluated in the true MDP. \red{remove performance}
% While having the overall learning curve trend match the ground truth is preferable, it is more important that the chosen ``best'' policy $\pi_{k^*}$ after applying early stopping achieves good performance. We computed the following quantitative measures on $\pi_{k^*}$: 

\section{Experiments \& Results} \label{sec:results}

In this section, we present empirical experiments that compare the utility of OPE methods for model selection in offline RL. In \cref{sec:results-arch}, we first consider the effectiveness of OPE methods in selecting neural network architectures and training hyperparameters. Specifically, we quantitatively measure the overall ranking quality, as well as the regret of selected policies. We also compare strategies for combining multiple OPE estimators, including naive averaging (of scores or ranking), WDR, as well as the proposed two-stage selection. To test the robustness of each OPE method in different settings and to enhance the generalizability of our findings, in \cref{sec:results-comparison}, we explore robustness to the additional OPE hyperparameters as well as various data conditions with different sample sizes or different behavior policies. The main experiments focus on the continuous state space setting; results for the discrete state space setting can be found in \cref{appx:early-stopping}.

% Note that WDR-AM is omitted for experiments with continuous state space due to its prohibitive computational requirement. 
% \emph{Depending on the claim you make in the paper, different components may be important for this section.}

\subsection{Model Selection of Neural Architectures \& Training Hyperparameters} \label{sec:results-arch}

When a function approximator (such as neural networks) is used, there are many hyperparameters related to architectures and training procedures that affect model complexity and can lead to overfitting. In this experiment, we consider learning policies on the sepsis simulator with a wide range of hyperparameters (\cref{tab:sepsis-HP-options}). Based on the hyperparameter combinations, we trained a total of 96 Q-networks using FQI on the training set, corresponding to 96 candidate policies, and then evaluated each candidate policy using OPE on validation data. 
%The experiment is repeated for 5 times, each with a different replication of training and validation data. 

\begin{table}[h]
    \centering
    \caption{Hyperparameter values considered for neural FQI on the sepsis simulator datasets. }
    \label{tab:sepsis-HP-options}
    \scalebox{0.9}{%
    \begin{tabular}{ll}
    \toprule
        \textbf{Hyperparameter} & \textbf{Values} \\
        \midrule
        Hidden layers & \{1, 2\} \\
        Hidden units & \{100, 200, 500, 1000\} \\
        Learning rate & \{1e-3, 1e-4\} \\
        FQI training iterations & \{1, 2, 4, 8, 16, 32\} \\
        \bottomrule
    \end{tabular}
    }
\end{table}

\subsubsection{Comparison Across OPE Methods}

To measure the ranking quality of each OPE method, we first inspect the scatter plots of OPE scores and true policy values (\cref{fig:sepsis-exp-HP} top). We quantify ranking quality using Spearman's rank correlation $\rho$ (reported as the means and standard deviations over 10 runs). In addition to overall ranking quality, we also measure regret@1, regret@5, and regret@10 to quantify how well model selection can help us choose the best (or a good) policy among the candidate set (\cref{fig:sepsis-exp-HP} bottom). 

\begin{figure}[h]
    \centering
    \includegraphics[scale=0.58]{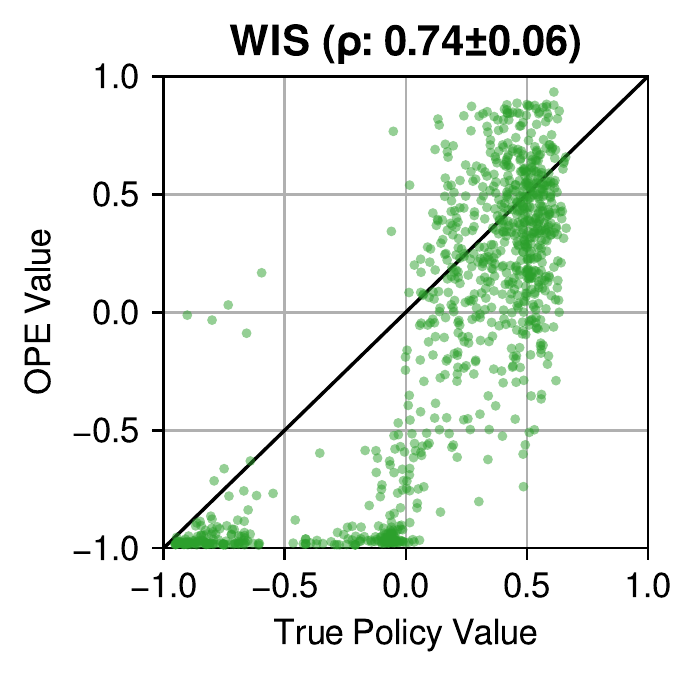}
    \includegraphics[scale=0.58]{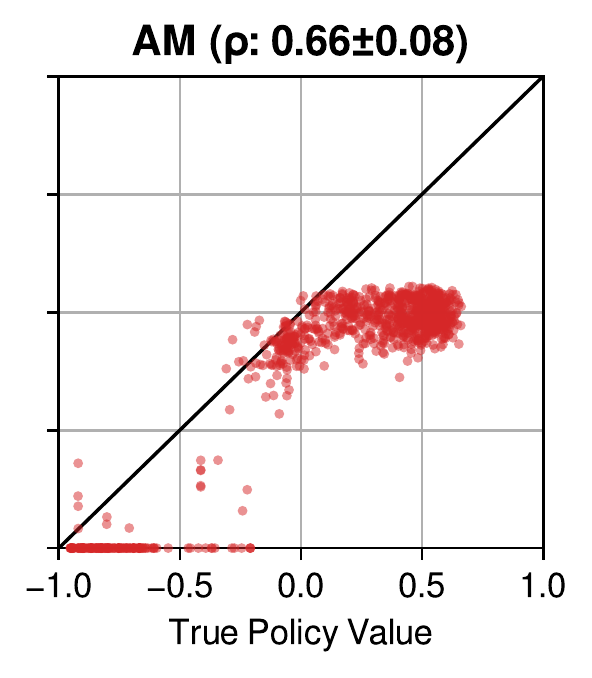}
    \includegraphics[scale=0.58]{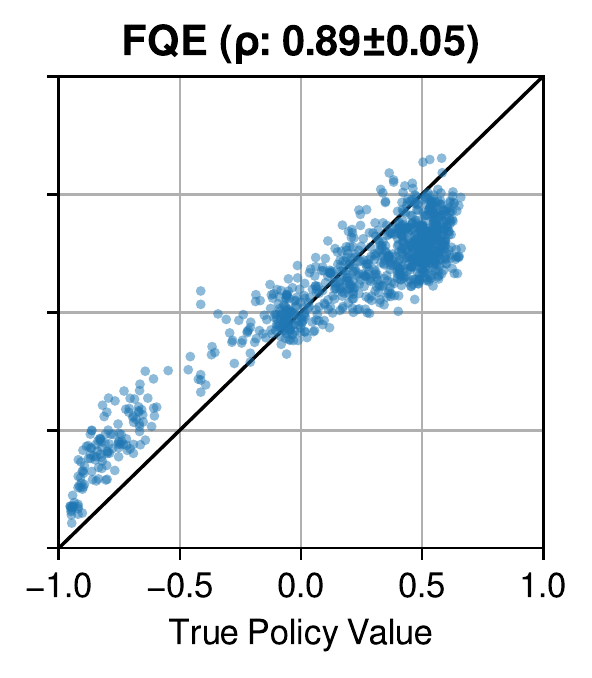}
    \includegraphics[scale=0.58]{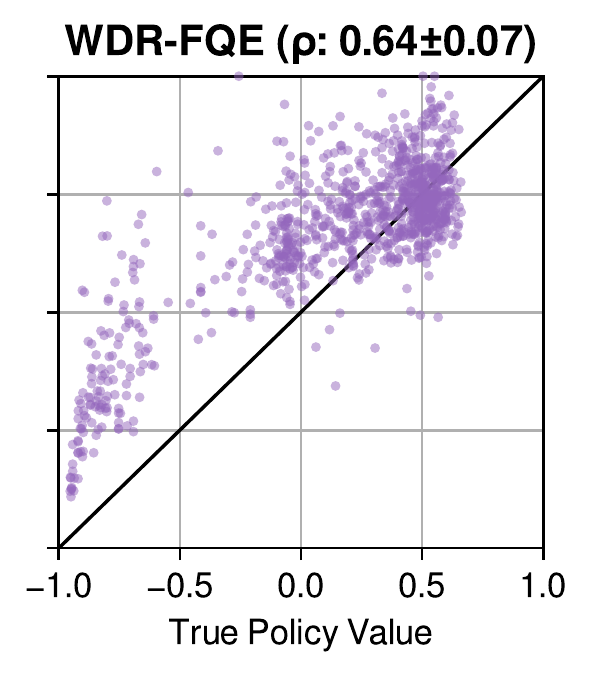}
    \includegraphics[scale=0.7,trim={0.5cm 0 0 -0.5cm }]{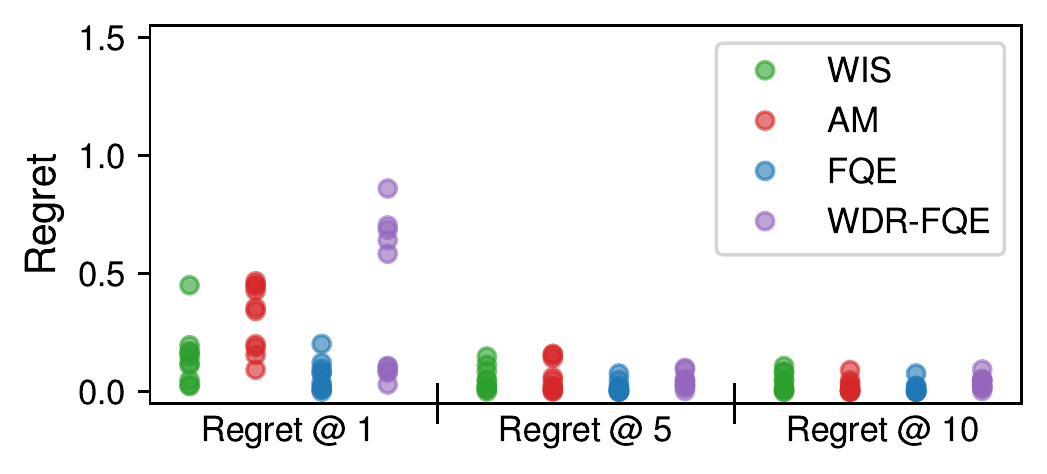}
    \caption{Top -- Scatter plots comparing true policy value with OPE validation score. The results from 10 runs (each with 96 points) are overlayed in each OPE subplot. All values are clipped to the range of possible returns $[-1,1]$ for clarity of visualization. Spearman's rank correlations are listed in the title of each subplot (mean $\pm$ std over 10 runs). FQE achieves the highest $\rho$ and produces the best ranking over candidate policies. Bottom -- Top-$n$ regret of different OPE methods in model selection compared to the best candidate policy. While the best candidate policy does not always have the highest OPE validation score, it is frequently among those with the top-5/top-10 validation performance. }
    \label{fig:sepsis-exp-HP}
\end{figure}

\paragraph{Results \& Discussion.} We observe that FQE produces the best ranking, with the highest $\rho = 0.89$ and points lying closest to the 45-degree line. WIS scores have a large variance and span almost the entire $[-1,1]$ range even for policies with a true value $>0$, but there appears to be good separation between policies with value $>0$ and those $<0$, leading to a relatively high $\rho = 0.74$. AM consistently underestimates the policy value, assigning scores close to $0$ for almost all policies, yet it maintains a reasonable overall ranking with $\rho = 0.66$. The hybrid method WDR-FQE, though commonly assumed to outperform its constituent OPE estimators (in terms of MSE), actually produces worse ranking than both WIS and FQE, resulting in a lower $\rho = 0.64$. The poor accuracy (contributing to poor ranking) of hybrid methods has also been observed for other RL tasks in past work \citep{voloshin2019OPE,fu2021benchmarks} and is likely due to various factors including limited data, environment stochasticity, estimated behavior policy, and model misspecification. We observe similar trends in regret@1 where FQE and WIS consistently achieves the lowest regret, while WDR-FQE sometimes leads to a very high regret ($>0.5$). However, all methods achieve near-zero regret for regret@5 and regret@10 and in general can identify the best candidate policy in the top-5/top-10 set. We also note that validation scores based on FQI value estimates or the temporal difference error produce poor ranking with $\rho<0.3$ and high regret (\cref{appx:HP-scatter}), highlighting the need to use OPE for estimating validation performance. 

\subsubsection{Combining Multiple OPE Estimators}

Given that WIS and FQE performed well, we explored various strategies of combining the two methods, in addition to WDR-FQE. As described in \cref{sec:combine-OPEs}, we apply the proposed two-stage approach where WIS is used to select a promising subset (containing policies with the top WIS scores). In practice, the subset size $\alpha$ needs to be pre-determined (for example, based on computational constraints), but here we used a range of different sizes and explored their effect. We also consider two naive approaches -- average score and average ranking of WIS and FQE -- as baselines. We computed regret and repeated the experiments for the same 10 replication runs as before (\cref{fig:sepsis-exp-HP-combine} left). 

\paragraph{Results \& Discussion.} Simple strategies such as average score and average ranking outperform the more sophisticated WDR-FQE, achieving lower regret compared to WIS (and comparable regret with FQE). However, these approaches rely on computing both OPE estimators on the entire candidate set and are thus computationally expensive. The two-stage approach achieves comparable or lower regret than FQE across different subset sizes, giving the lowest regret at a subset size of 24 (and 48). Moreover, with minimal (but arguably necessary) additional computation compared to using WIS alone, we can achieve lower regret using the two-stage approach compared to using FQE alone, with potentially greater savings in computational costs (\cref{fig:sepsis-exp-HP-combine} right). 

\begin{figure}[h]
    \centering
    \begin{tabular}{cc}
    \includegraphics[scale=.725,valign=m]{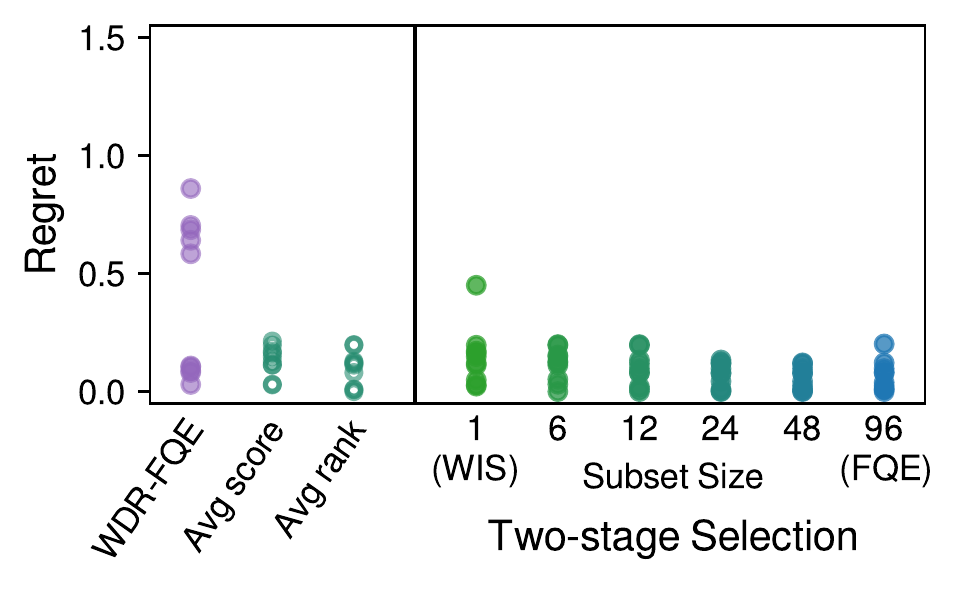}
    &
    \includegraphics[scale=.725,valign=m]{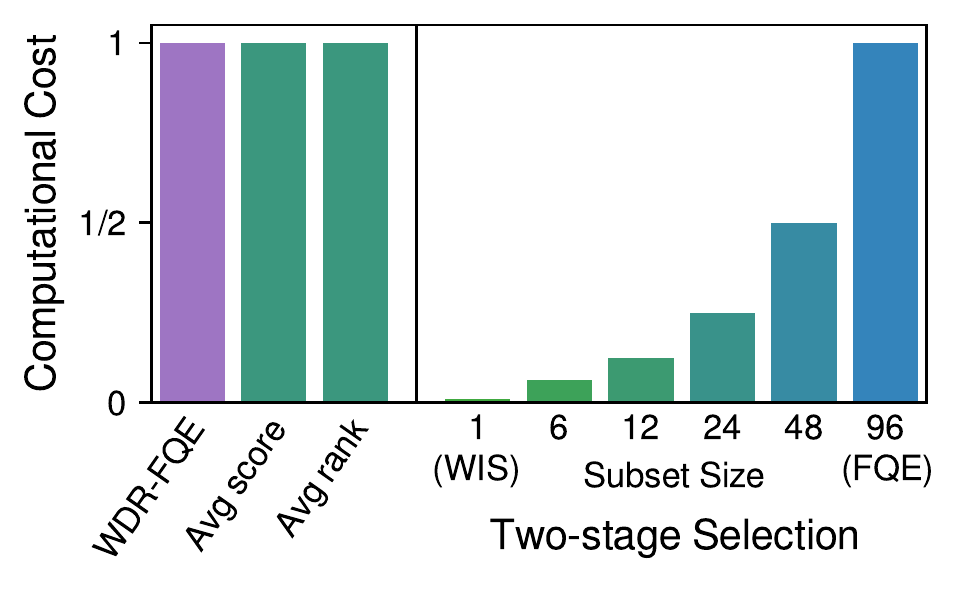}
    \end{tabular}
    \caption{Comparison of regret (left) and relative computational cost (right) of different strategies of combining WIS and FQE for model selection. The computational costs are normalized with respect to a complete FQE evaluation over the entire candidate set ($K=96$) and are for illustration purposes only. Two-stage selection with an intermediate subset size (e.g., 24) provides the best balance between performance and computational costs. }
    \label{fig:sepsis-exp-HP-combine}
\end{figure}

\subsection{Sensitivity Analyses} \label{sec:results-comparison}
In the experiments above, we made several assumptions about the validation data: (i) we used default auxiliary hyperparameters, (ii) we assumed a uniformly random behavior policy, and (iii) our validation data had a large sample size. These assumptions might not hold in practice. In this section, we vary the experimental conditions along these dimensions to test the generalizability of our conclusions. For these experiments, we kept the training set and training procedures the same throughout this comparison to generate the same candidate set for each run, so that the only difference comes from different validation data. We also kept the same neural network architecture and training procedures for auxiliary models. These experiments are repeated for the same 10 replication runs. We focus our results on the three non-hybrid OPE methods (WIS, AM, and FQE) as well as the proposed two-stage approach combining WIS with FQE with a subset size of 24. 
% because WDR resulted in large variances and heavily skewed the plots.
% (results for WDR-AM and WDR-FQE can be found \cref{appx:results})
Additional sensitivity analyses for both discrete and continuous state spaces can be found in \cref{appx:early-stopping}. 

\begin{figure}[h]
    \centering
    \includegraphics[scale=0.7,trim={1.cm 0 0 0}]{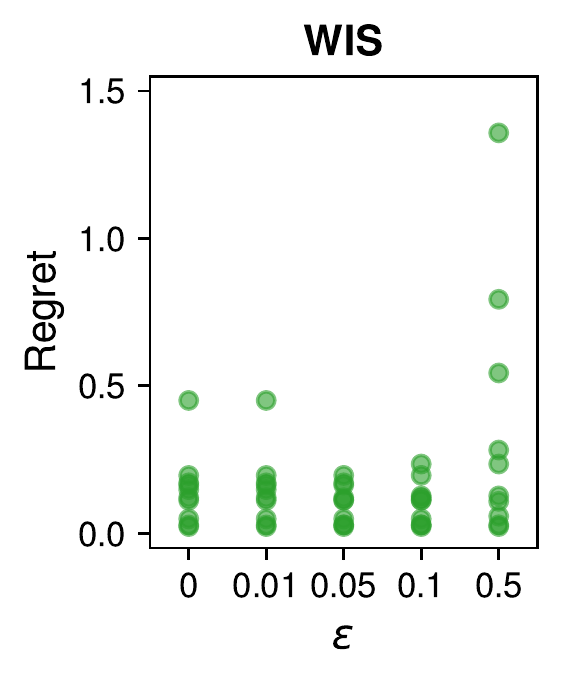}  \includegraphics[scale=0.7,trim={1.5cm 0 0 0},clip]{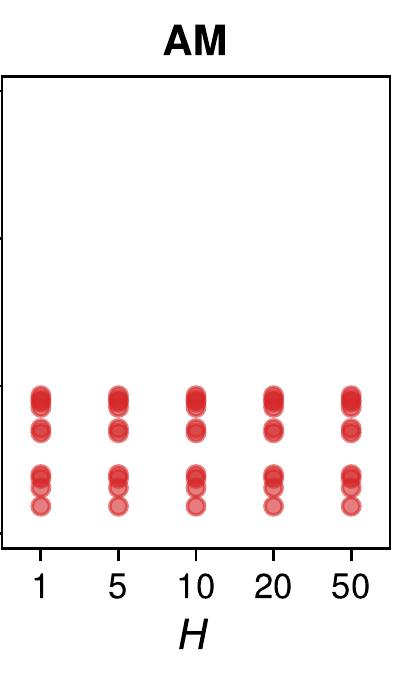} \includegraphics[scale=0.7,trim={1.5cm 0 0 0},clip]{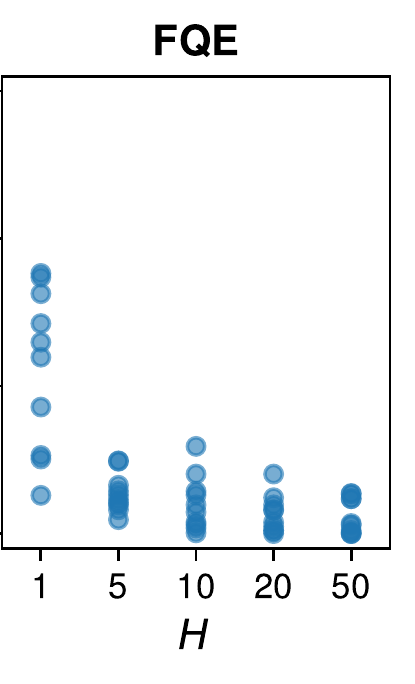} 
    \caption{Regret comparisons for model selection with different auxiliary OPE hyperparameters. AM is the most robust to $H$ but underperforms the other two OPE methods. WIS and FQE produce low regret given reasonable $\varepsilon$ or $H$ values. }
    \label{fig:sepsis-exp4}
\end{figure}

\subsubsection{Sensitivity to OPE Hyperparameters} \label{sec:results-HP}

Since each OPE method has its own auxiliary hyperparameters, it is important to understand how they may affect OPE validation ranking quality, especially as there is no effective way of selecting them. In this experiment, we focus on the hyperparameters that must be set \textit{a priori}, namely: the policy softening parameter $\varepsilon$ for WIS and the evaluation horizon $H$ for AM/FQE. We vary these hyperparameters in the reasonable ranges and measure the corresponding regret following model selection. For continuous state space models, even though the auxiliary models contain hyperparameters themselves (e.g., the neural network architecture), these can be tuned via standard model selection within each auxiliary supervised learning task; for consistency, we kept these hyperparameters the same throughout. 
%  and modify the evaluation problem

\paragraph{Results and Discussion.} We visualize the regret in \cref{fig:sepsis-exp4}. Among the three OPE methods, AM appears to be the most robust to different $H$ values, although the regret it obtains is higher than the other two OPE methods and does not decrease as $H$ increases. For WIS, the lowest regret is achieved at $\varepsilon=0.1$, but using $\varepsilon$ values that are too small (potential issues with effective sample size) or too large ($\tilde{\pi}$ is too far from $\pi$) led to regret with slightly larger variance. FQE achieves more stable regret at larger evaluation horizons, e.g., $H \geq 10$, that correspond to a closer approximation to a full, infinite-horizon evaluation. Although WIS and FQE are more sensitive to their auxiliary hyperparameters, the performance is optimal (or close to optimal) following existing rules of thumb and recommendations we provided in \cref{sec:applyOPE}. 

% \begin{figure}[h]
%     \centering
%     \includegraphics[width=.7\linewidth]{figs/plot-legend.pdf}
%     \includegraphics[scale=0.8]{figs/sepsis-tab-exp2-size.pdf} \quad \includegraphics[scale=0.8]{figs/sepsis-cont-exp2-size.pdf}
%     \caption{OPE on varying validation sizes}
%     \label{fig:sepsis-exp2}
% \end{figure}

\subsubsection{Sample Size of Validation Data} \label{sec:results-size}
In our main experiments, we used validation data with 10,000 episodes, which is as large as the training data. In supervised learning, the size of validation data is more commonly only a fraction of the training data size. In this comparison, we reduced the amount of validation data used for OPE and measured the regret. 

\begin{figure}[h]
    \centering
    \includegraphics[valign=t,scale=0.75,trim={0 0 0 0.675cm},clip]{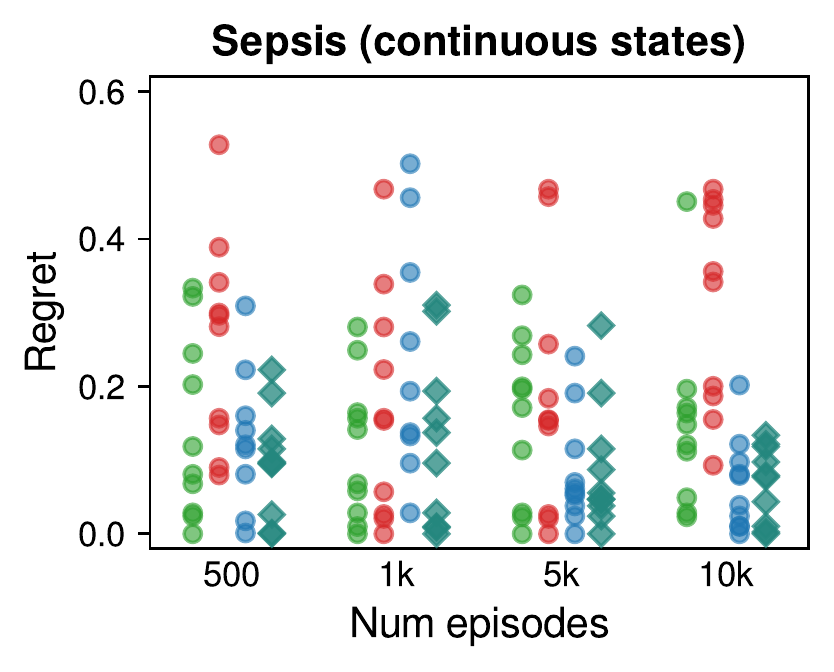} \quad
    \includegraphics[valign=t,scale=0.75,trim={0 0 0 0.675cm},clip]{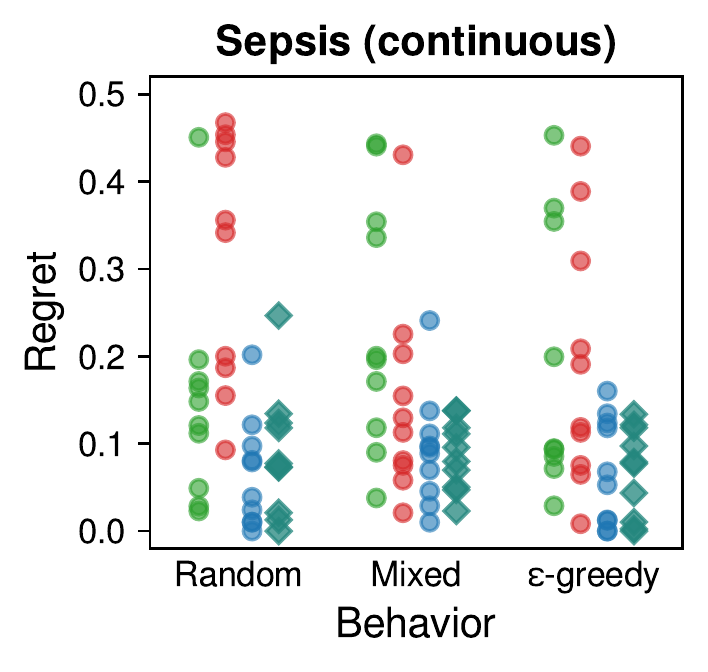} \includegraphics[valign=t,scale=.8]{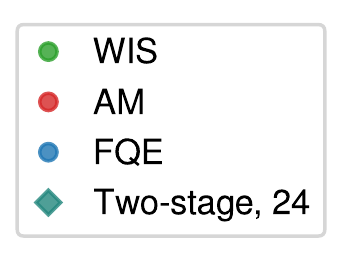}
    \caption{Regret comparison for OPE using validation data with varying sample sizes and varying behaviors. Left -- A larger validation set generally results in lower regret across all OPE methods. Right -- Less exploratory behavior can lead to higher regret. In almost all scenarios, the proposed two-stage approach is able to maintain low regret while reducing overall computational cost. }
    \label{fig:sepsis-exp2-exp3}
\end{figure}

\paragraph{Results and Discussion.}
All OPE methods tend to benefit from larger validation sizes, achieving lower regret (\cref{fig:sepsis-exp2-exp3}, left). The trend is less apparent for WIS (which generally results in noisier value estimates) and for AM (possibly due to model misspecification). The proposed two-stage selection (combining WIS and FQE with a subset size of 24) consistently matches (and sometimes slightly outperforms) FQE in terms of regret, except for very small sample sizes (e.g., 500 episodes). This could be because WIS is more susceptible to large variances due to limited sample sizes and produces less reliable pruning in the first stage. Given reasonably sized validation data (1,000 episodes or more for this problem), the proposed two-stage approach maintains performance comparable to FQE, while reducing computational costs. 

\subsubsection{Type of Behavior \& Extent of Exploration} \label{sec:results-behavior}
The behavior policies used to collect validation data can affect which part of the state-action space receives sufficient exploration and the resultant OPE model selection performance. While we previously assumed a uniformly random policy that provides sufficient exploration, this is rarely the case in real life; e.g., clinicians do not behave randomly when selecting treatments. Moreover, historical data may be generated by multiple clinicians who use different policies. Here, we examine how the behavior policy affects model selection. We first considered validation data containing 10,000 episodes following an $\epsilon$-greedy behavior policy (with $\epsilon=0.1$); though this produces near-optimal behavior, it limits the extent of exploration. We also considered data that contained 5,000 episodes of both the random policy and the $\epsilon$-greedy policy to create a mixture of behaviors (with the same total sample size of 10,000 episodes). Although these two settings violate the assumption that training data and validation data follow the same data distribution (as discussed in \cref{sec:pipeline}), the evaluation problem is still valid since all data are drawn from the same underlying MDP environment. 

% A common assumption is that the validation data were generated using a single behavior policy \citep{raghu2017continuous,komorowski2018AI_Clinician} even though it may be hard to verify in real problems (e.g., when historical data are generated by different clinicians).
%\footnote{In experiments we will explore settings when this assumption is violated, i.e., when data follows a mixture of multiple behavior policies but our $\hat{\pi}_b$ estimation assumes a single behavior policy. }

\paragraph{Results and Discussion.}
The random behavior policy produces more exploratory data compared to an $\epsilon$-greedy policy, leading to slightly lower regret in general (\cref{fig:sepsis-exp2-exp3}, right). Surprisingly, even when a mixture of multiple behavior policies is used, WIS is not heavily affected by this policy misspecification and still leads to low regret. Reassuringly, the proposed two-stage selection (with a subset size of 24) is able to maintain (or outperform) the regret achieved by FQE while reducing overall computation, even in challenging situations with less exploratory data or mixed behaviors.

% FQE and AM remain relatively stable across the different behavior settings and generally lead to the lowest regret with smaller variance. 

% Another important consideration is the size of validation set relative to the training set (10k in all cases). While in supervised learning it is typical the validation set is much smaller than the training set, in our offline RL setting, using a large validation set (e.g., as large as the training set) can help alleviate the negative effects of overfitting to the greatest extent. %This effect is also present although slightly weaker for WDR methods. 

% This tends to have overall positive effect for selecting a well performing policy, and thus a separate validation set might not be necessary. 

% Since we are uncertain the extent of implicit regularization for OPE model fitting in general, a separate validation set is still recommended (and is necessary if a tabular representation is used). 

% this allows us to use more data for training (and potentially learn better policies), it could cause the validation performance to be overconfident and lead to overfitting. 

\section{Conclusion}
Model selection for offline RL is challenging because we cannot easily obtain the equivalent of ``validation performance'' given only historical interaction data. To date, RL for healthcare research has skirted the issue of model selection by either assuming default hyperparameters or relying on ground-truth performance only obtainable with simulators. In this work, we explored a model selection framework for offline RL that uses OPE as a proxy estimate for validation performance. We provided an in-depth analysis of how this can be implemented using four popular OPE methods, detailing the practical considerations such as the reliance on auxiliary hyperparameters/models and computational requirements. On a simulated RL task of sepsis treatment, we found that FQE consistently provides the most accurate ranking over candidate policies and is the most robust to various data conditions (sample sizes and behavior policies). Given the high computational cost of FQE, we proposed a simple two-stage approach to combine two OPE estimators and demonstrated that it outperforms simple averaging or more sophisticated doubly-robust methods, while avoiding expensive computation on low-quality policies. 

There has been increasing interest on the problem of model selection for offline RL in recent years (\cref{appx:related-work}). Recently, \citet{paine2020hyperparameter} demonstrated that FQE is a useful strategy for hyperparameter selection in offline RL on several benchmark control tasks. In contrast, we compare multiple OPE estimators both qualitatively and quantitatively to address questions of practical importance for healthcare settings. Our empirical results further corroborate the findings of \citet{irpan2019OPC} and \citet{paine2020hyperparameter}, who showed that FQI values and TD errors are poor validation metrics due to overestimation of value functions during training, thus highlighting the importance and necessity of using OPE to estimate validation performance. During the preparation of this work, \citet{fu2021benchmarks} published a benchmark suite for comparing OPE methods on continuous control tasks, with Spearman's correlation as one of evaluation metrics. Similar to our empirical findings, they found FQE to consistently produce the best results, whereas other OPE approaches (including AM and WIS) varied more with respect to the specific task and modeling choice. However, their focus is on presenting a standardized benchmark to compare OPE methods, rather than carefully reflecting on the practical and implementation issues that arise when using OPE for model selection. 

We note several important limitations of this work. First, our experiments are performed on a simulated RL task because of the need to obtain true performance of learned policies. Though we explored different data conditions including behavior policies and sample sizes, our simulation might not capture all nuances in real data collection. In particular, certain properties of the simulation (discrete action spaces, short horizons, noiseless observations, no missingness, sparse terminal rewards) may limit the generalizability of our conclusions to real observational data and other RL problems. Using a simulator also does not address important issues that may arise when designing the RL problem setup based on real datasets (e.g., defining a good action space or meaningful initial states). We encourage future work to investigate more healthcare RL problems using our released code. Though the simulated sepsis task may not directly produce any clinically relevant insights for treating sepsis, we expect future work on sepsis management \citep{komorowski2018AI_Clinician,raghu2017deepRL,raghu2017continuous,killian2020representation} to build upon our lessons and improve their model selection procedure, thus enabling fairer comparisons. Second, we only considered four OPE methods that are commonly used in healthcare tasks. It could be beneficial to understand the pros and cons of other OPE methods when incorporated into the model selection framework, e.g., approaches based on marginalized importance sampling \citep{liu2018breaking,zhang2019gendice,uehara2020minimax}. Some of these techniques have been compared quantitatively in concurrent work by \citet{fu2021benchmarks} on a continuous control benchmark suite, but none of them outperformed FQE in terms of ranking quality. Third, we assumed the same validation data is used to fit the auxiliary models and calculate OPE estimates. It is possible to use two separate validation sets (one for fitting models, the other for applying OPE); we preferred our current strategy as it leverages more data for both steps and empirically produced reasonable results, though this may violate the independence assumptions needed for the theoretical guarantees of OPE \citep{jiang2016doubly}. We also did not implement additional ``tricks'' that could help improve OPE accuracy, e.g., truncated importance weights for WIS \citep{ionides2008truncated,swaminathan2015self} or more sophisticated models for AM \citep{torabi2018cloning,yu2020MOPO,kidambi2020MOReL}. Lastly, we generated datasets with sufficient exploration and coverage over the entire state-action space, and trained all policies using a vanilla RL algorithm (i.e., FQI) that does not address distribution shift. Recently proposed algorithms that specifically account for distribution shift \citep{kidambi2020MOReL,yu2020MOPO,kumar2020CQL,liu2020MBS} require additional hyperparameters, and we expect our framework to be useful for selecting these hyperparameters if the OPE step also considers distribution shift (e.g., by restricting the state-conditional action space to disallow rare actions). Despite these limitations, our empirical results suggest that using OPE for model selection in offline RL is a promising solution. 

In this paper, we outlined important practical considerations when using OPE for model selection in offline RL. Our work serves as a guide for RL practitioners in healthcare and other high-stakes domains when considering the model selection step, and can ultimately lead to better policies on real-world datasets. We encourage researchers in RL for health to build upon our lessons and report model selection techniques, in an effort to improve reproducibility, enable fair comparisons, and advance the field.

% The OPE validation pipeline can also be used to select offline RL specific algorithms often involve additional hyperparameters 
% We find an interesting trade-off between computation vs ranking quality, suggesting combining them. For example, use WIS to narrow the candidate set by half and then FQE for a more reliable ranking. 
% the ranking correlates quite well with true policy performance, suggesting OPE is a useful solution. However, we also noted several factors affecting, and these depend on the OPE methods and the RL setup. Additionally we analyzed computational requirements. 
% {\shengpu summarize technical analysis and empirical results, overall takeaway, recommendations, and clinical implications}
% and can potentially be combined to select the hyperparameters for recently proposed offline RL algorithms 

% ACKNOWLEDGEMENTS ONLY GO IN THE CAMERA-READY, NOT THE SUBMISSION
\acks{This work was supported by the National Science Foundation (NSF award no. IIS-1553146) and the National Library of Medicine (NLM grant no. R01LM013325). The views and conclusions in this document are those of the authors and should not be interpreted as necessarily representing the official policies, either expressed or implied, of the National Science Foundation nor the National Library of Medicine. This work was supported in part through computational resources and services provided by \href{https://arc.umich.edu/}{Advanced Research Computing} at the University of Michigan, Ann Arbor. The authors would like to thank \href{mailto:admodi@umich.edu}{Aditya Modi} and members of the \href{https://wiens-group.engin.umich.edu/}{MLD3 group} for helpful discussions regarding this work, as well as the reviewers for their constructive feedback.}

\bibliography{ref}

\clearpage
\appendix

\section{Pseudocode for FQI \& FQE} \label{appx:algs}
Below we present the pseudocode for FQI (\cref{alg:FQI}) and FQE (\cref{alg:FQE}), using notation we introduced in \cref{sec:background}.
\begin{figure}[H]
\centering
\scalebox{0.67}{%
\begin{minipage}{0.7\linewidth}%
\begin{algorithm}[H]
  \caption{Fitted Q Iteration \\ \hspace*{\algorithmicindent} \citep{ernst2005batchRL,riedmiller2005NFQ}}\label{alg:FQI}
  \begin{algorithmic}[1]
    \State \textbf{Input:} dataset $\mathcal{D}$, 
    \Statex[1] function class $\mathcal{F}$, max iteration $H$
    \State \textbf{Initialize} $\tilde{Q}_0(s_i,a_i)=0$ for all $(s_i, a_i) \in \mathcal{D}$
    \For{$h$ \textbf{in} $1 \dots H$}
        \State \text{Compute target} $y_i = r_i + \gamma \max_{a'\in\mathcal{A}}\tilde{Q}_{h-1}(s'_i, a') \ \forall i$
        \State \text{Build training set} $\tilde{\mathcal{D}}_h = \{(s_i, a_i), y_i\}_{i=1}^{N}$
        \State \text{Solve a supervised learning problem} 
        \Statex[2] $\tilde{Q}_{h} \gets \argmin_{f \in \mathcal{F}} \sum_{i=1}^{N} (f(s_i,a_i) - y_i)^2$
    \EndFor
    \State \textbf{Output:} $\pi(s) = \argmax_{a\in\mathcal{A}} \tilde{Q}_H(s,a)$
  \end{algorithmic}
\end{algorithm}
\end{minipage}
}%
\quad 
\scalebox{0.67}{%
\begin{minipage}{0.7\linewidth}%
\begin{algorithm}[H]
  \caption{Fitted Q Evaluation \\ \hspace*{\algorithmicindent} \citep{le2019FQE,voloshin2019OPE}}\label{alg:FQE}
  \begin{algorithmic}[1]
    \State \textbf{Input:} dataset $\mathcal{D}$, policy to be evaluated $\pi$, 
    \Statex[1] function class $\mathcal{F}$, max iteration $H$
    \State \textbf{Initialize} $\tilde{Q}_0(s_i,a_i)=0$ for all $(s_i, a_i) \in \mathcal{D}$
    \For{$h$ \textbf{in} $1 \dots H$}
        \State \text{Compute target} $y_i = r_i + \gamma \sum_{a'\in\mathcal{A}}\pi(a'|s'_i) \tilde{Q}_{h-1}(s'_i, a') \ \forall i$
        \State \text{Build training set} $\tilde{\mathcal{D}}_h = \{(s_i, a_i), y_i\}_{i=1}^{N}$
        \State \text{Solve a supervised learning problem} 
        \Statex[2] $\tilde{Q}_{h} \gets \argmin_{f \in \mathcal{F}} \sum_{i=1}^{N} (f(s_i,a_i) - y_i)^2$
    \EndFor
    \State \textbf{Output:} $\tilde{Q}_H$
  \end{algorithmic}
\end{algorithm}
\end{minipage}
}%
\end{figure}

\section{OPE Computational Requirements} \label{appx:OPE-compute}
As we have mentioned in \cref{sec:applyOPE}, each OPE estimator depends on auxiliary models. For example, WIS requires a model of the behavior policy $\hat{\pi}_b$, whereas AM requires the transition/reward models $\hat{p}$ and $\hat{r}$. Here, we further discuss how these models can be obtained from data, for RL problems with either a discrete or continuous state space. We assume a discrete action space and leave continuous action spaces to future work. 

\paragraph{Discrete state space.} This corresponds to the tabular setting, where the ``models'' are probability/value tables and can be obtained via maximum likelihood estimation (MLE) using count-based averaging:
\begin{align*}
\hat{\pi}_b(a|s) &= \frac{\operatorname{count}(s,a)}{\operatorname{count}(s)} &
\hat{p}(s'|s,a) &= \frac{\operatorname{count}(s,a,s')}{\operatorname{count}(s,a)} \\
\hat{r}(s,a) &= \frac{\sum_{i=1}^{N}r_i\mathbbm{1}_{(s_i=s, a_i=a)}}{\operatorname{count}(s,a)} &
\hat{\mu}_0(s) &= \frac{\operatorname{count}(s_1=s)}{m}
\end{align*}
where $\operatorname{count}(s,a)$ and $\operatorname{count}(s,a,s')$ are the (marginal) counts of the state-action pair and state-action-next-state transition in $\mathcal{D}$ respectively. These quantities can be found relatively efficiently by looping through the dataset once and accumulating the respective counts. 

\paragraph{Continuous state space.} In this case, a state is represented as a feature vector $\mathbf{x}(s) \in \mathbb{R}^d$. Thus, all of the above-mentioned models need to be approximated using some parameterized functions (e.g., neural networks). These models (as shown in \cref{tab:OPE-models}) are independent from the main policy learning process; they can be treated as classification or regression tasks and trained using standard supervised learning techniques (e.g., stochastic gradient descent with early stopping on a validation split). Furthermore, each of these models may be used for inference for multiple times given different inputs. We summarize the total number of additional models and inference runs for each OPE in \cref{tab:OPE_computation}. 

\begin{table}[h]
    \centering
    \caption{Details of auxiliary supervised models required by each OPE for problems with continuous state spaces. }
    \label{tab:OPE-models}
    \scalebox{0.85}{
    \begin{tabular}{ccrlcc}
    \toprule
    \textbf{OPE} & \textbf{Model} & \multicolumn{2}{l}{\hspace*{2.25em}\textbf{Pattern Set}}  & \textbf{Task Type} & \textbf{Loss} \\
    & & $\{x_i$,& $y_i\}_{i=1}^{N}$ & & \\
    \midrule
    WIS & $\hat{\pi}_b(a|s)$ & $\mathbf{x}(s_i)$ & $a_i$ & classification & cross entropy \\
    \noalign{\vskip 3pt}
    \hline
    \noalign{\vskip 3pt}
    \multirow{2}{*}{AM}& $\hat\Delta_s(s,a)$ & $\langle\mathbf{x}(s_i), a_i\rangle$ & $\mathbf{x}(s'_i)-\mathbf{x}(s_i)$ & regression & mean squared error \\
    & $\hat{r}(s,a)$ & $\langle\mathbf{x}(s_i), a_i\rangle$ & $r_i$ & regression & mean squared error \\
    \noalign{\vskip 3pt}
    \hline
    \noalign{\vskip 6pt}
    FQE & $\tilde{Q}_{h}(s,a)$ & $\langle\mathbf{x}(s_i), a_i\rangle$ & $\displaystyle r_i+\max_{a'\in\mathcal{A}}\tilde{Q}_{h-1}(s'_i,a')$ & regression & mean squared error \\
    \noalign{\vskip 1pt}
    \bottomrule
    \end{tabular}
    }%
\end{table}

\begin{table}[h]
    \centering
    \caption{Summary of computational requirements of OPE used for model selection for $K$ candidate policies. We assume the dataset consists of $m$ episodes and $N$ transitions, and $H$ is the length of the evaluation horizon. For WDR we only list additional computation that are not already computed in other OPE estimators. }
    \label{tab:OPE_computation}
    \scalebox{0.85}{
    \begin{tabular}{ccccc}
    \toprule
    \multirow{2}{*}{\textbf{OPE}} & \multicolumn{2}{c}{\textbf{Fitting}} & \multicolumn{2}{c}{\textbf{Inference}} \\
     & Runs & Samples & Runs & Samples  \\
    \midrule
    WIS & 1 & $N$ & $1$ & $N$ \\
    AM & 2 & $N$ & $2K$ & $mH$ \\
    FQE & $KH$ & $N$ & $KH$; $K$ & $N$; $m$ \\
    WDR-AM & --- & --- & $2|\mathcal{A}|^{H} HK$ & $N$ \\
    WDR-FQE & --- & --- & $K$ & $N$ \\
    \bottomrule
    \end{tabular}
    }%
\end{table}

\section{Sepsis Simulator - Details} \label{appx:sepsisSim}

We use the sepsis simulator adapted from \citet{oberst2019gumbel} and \citet{futoma2020popcorn}, which is crudely modeled after the physiology of patients with sepsis. Our experiments are based on an implementation publicly available at \url{https://github.com/clinicalml/gumbel-max-scm/tree/sim-v2/sepsisSimDiabetes}. 
\begin{itemize}
    \item \textbf{Action space.} There are 8 actions based on combinations of 3 binary treatments: antibiotics, vasopressors, and mechanical ventilation, where each treatment can affect certain vital signs and may raise/lower their values with pre-specified probabilities. 
    \item \textbf{State space.} The underlying patient state consists of 5 discrete variables: a binary indicator for diabetes status, and 4 ordinal-valued vital signs (heart rate, blood pressure, oxygen concentration, glucose). The previously administered treatments are also encoded in the state as 3 additional binary variables. See \cref{tab:sepsisSim-states} for more details. We consider two state representations: a discrete state space with $|\mathcal{S}| = 1,440$, as well as a feature vector $\mathbf{x}(s) \in \{0,1\}^{21}$ that uses a one-hot encoding for each underlying variable. We separately added two absorbing states to represent discharge and death. 
    \item \textbf{Initial states distribution.} The diabetes indicator is set to 1 w.p. $0.2$. To make the problem easier, the MDP can start from any non-terminating state (including those with treatments enabled) with equal probability (conditioned on the diabetes indicator value), for a total of $303 \times 2 = 606$ initial states. 
    \item \textbf{Transition dynamics.} The treatments (either added or withdrawn) affect different state variables and are applied sequentially as detailed in \cref{tab:sepsisSim-transition}. 
    \item \textbf{Rewards \& termination condition.} A patient is discharged only when all vitals are normal and all treatments have been stopped; death occurs if 3 or more vitals are abnormal. Rewards are sparse and only assigned at the end of each episode (when transitioning from a terminating state to one of the two absorbing states), with $+1$ for survival and $-1$ for death, after which the system transitions into the respective absorbing state and obtains $0$ reward afterwards. 
\end{itemize}

\begin{table}[h]
    \centering
    \caption{Details on state variables.}
    \label{tab:sepsisSim-states}
    \scalebox{0.9}{%
    \begin{tabular}{llcc}
    \toprule
    \textbf{Variable} & \textbf{Name} & \multicolumn{2}{c}{\textbf{Levels (numeric / text)}} \\
    \midrule
    hr & Heart Rate & \{0,1,2\} & \{L,N,H\} \\
    sbp & Systolic Blood Pressure & \{0,1,2\} & \{L,N,H\} \\
    o2 & Percent Oxygen & \{0,1\} & \{L,N\} \\
    glu & Glucose & \{0,1,2,3,4\} & \{LL,L,N,H,HH\} \\
    abx & Antibiotics & \{0,1\} & \{off,on\} \\
    vaso & Vasopressors & \{0,1\} & \{off,on\} \\
    vent & Mechanical Ventilation & \{0,1\} & \{off,on\}\\
    diab & Diabetes Indicator & \{0,1\} & \{no,yes\} \\
    \bottomrule
    \end{tabular}
    }
\end{table}

\begin{table}[h]
\centering
\caption{Details on transition dynamics.}
\label{tab:sepsisSim-transition}
\scalebox{0.9}{%
\begin{tabular}{cc|ccc|cl}
\toprule
\textbf{Step} & \textbf{Variable} & \textbf{Current} & \textbf{New} & \textbf{Change} & \multicolumn{2}{c}{\textbf{Effect}} \\
\hline
\multirow{3}{*}{1} & \multirow{3}{*}{abx} & -- & on & -- & \makecell[c]{hr \\ sbp} & \makecell[l]{H $\rightarrow$ N w.p. 0.5 \\ H $\rightarrow$ N w.p. 0.5} \\
\cline{3-7}
& & on & off & withdrawn & \makecell[c]{hr \\ sbp} & \makecell[l]{N $\rightarrow$ H w.p. 0.1 \\ N $\rightarrow$ H w.p. 0.5} \\
\hline
\multirow{2}{*}{2} & \multirow{2}{*}{vent} & -- & on & -- & o2 & L $\rightarrow$ N w.p. 0.7 \\
\cline{3-7}
& & on & off & withdrawn & o2 & N $\rightarrow$ L w.p. 0.1 \\
\hline
\multirow{7}{*}{3} & \multirow{7}{*}{vaso} & \multirow{3}{*}{--} & \multirow{3}{*}{on} & \multirow{3}{*}{--} & sbp & \makecell[l]{L $\rightarrow$ N w.p. 0.7 (non-diabetic) \\ N $\rightarrow$ H w.p. 0.7 (non-diabetic) \\ L $\rightarrow$ N w.p. 0.5 (diabetic) \\ L $\rightarrow$ H w.p. 0.4 (diabetic) \\ N $\rightarrow$ H w.p. 0.9 (diabetic)} \\
\cline{6-7}
& & & & & glu & \makecell[l]{LL $\rightarrow$ L,  L $\rightarrow$ N, N $\rightarrow$ H, \\ H $\rightarrow$ HH w.p. 0.5 (diabetic)} \\
\cline{3-7}
& & on & off & withdrawn & sbp & \makecell[l]{N $\rightarrow$ L w.p. 0.1 (non-diabetic) \\ H $\rightarrow$ N w.p. 0.1 (non-diabetic) \\ N $\rightarrow$ L w.p. 0.05 (diabetic) \\ H $\rightarrow$ N w.p. 0.05 (diabetic)} \\
\hline
\makecell[c]{4 \\ 5 \\ 6 \\ 7} & \makecell[c]{hr \\ sbp \\ o2 \\ glu} & \multicolumn{3}{c}{fluctuate} & \multicolumn{2}{|l}{\makecell[l]{vitals spontaneously fluctuate when not affected \\ by treatment (either enabled or withdrawn) \\ \textbullet\ the level fluctuates $\pm 1$ w.p. 0.1, except: \\ \textbullet\ glucose fluctuates $\pm 1$ w.p. 0.3 (diabetic) }} \\
\bottomrule
\end{tabular}
}
\end{table}

% For data generation, episodes are truncated at a maximum length of 20. We used two types of behavior policies: a uniformly random policy that provides as much exploration as possible, and a near-optimal $\epsilon$-greedy policy with $\epsilon=0.10$ probability taking a random action in order to balance good behavior performance and the amount of exploration. For each behavior policy, we generated 10 pairs of datasets (for training and validation), each with large number of episodes ($m=10,000$) and a different random seed. A more detailed description of the simulation setup can be found in \cref{appx:sepsisSim} and the source code. 

\section{Additional Experimental Results} \label{appx:results}

\subsection{Model Selection using FQI Values \& TD Errors} \label{appx:HP-scatter}

Following \cref{sec:results-arch}, we consider two additional metrics to be used as a proxies of validation performance. 
\begin{itemize}
    \item FQI values. This is based on the predicted values from the Q-network learned by FQI. It is in a way similar to ``training performance'' in the supervised learning setting. Similar to all OPE scores, we compute FQI score with respect to the empirical initial state distribution (from the validation set): 
    \[ \hat{v}_{\FQI}(\pi) = \frac{1}{m} \sum_{j=1}^{m} \widehat{V}^{\pi}_{\FQI}(s_1^{(j)}), \text{ where } \widehat{V}^{\pi}_{\FQI}(s) = \sum_{a\in\mathcal{A}} \pi(a|s) \widehat{Q}^{\pi}_{\FQI} (s, a). \]
    \item RMS-TDE, i.e., root-mean-squared temporal difference error \citep{dann2014TD}, is the objective function (i.e., ``loss function'') used in temporal difference learning algorithms such as FQI (more precisely, the loss function is MS-TDE without the square root). This is commonly calculated using the empirical distribution of transitions from the validation set. 
    \[ \ell = \sqrt{\mathbb{E}_{(s_i, a_i, r_i, s'_i)\sim \mathcal{D}}[\delta_i^2]} \ ,\] where \( \delta_i = \left[r_i + \gamma \sum_{a'} \pi(a'|s'_i) \widehat{Q}(s'_i,a')\right] - \widehat{Q}(s_i, a_i) \) defines the temporal difference error of a transition tuple, and $\widehat{Q}$ is the same $\widehat{Q}^{\pi}_{\FQI}$ from FQI training. 
\end{itemize}
We evaluate how well FQI values and RMS-TDE can serve as ``validation scores'' for model selection by measuring the ranking quality and regret (\cref{fig:sepsis-exp-HP-appx}). We observe these two ``validation scores'' clearly underperforms the OPE methods considered in the main paper, achieving Spearman's rank correlation $|\rho|<0.3$ with very little trend in the scatter plots, and high regret even when considering the top-5/top-10 set. This suggests that FQI values and RMS-TDE are not effective metrics for model selection in offline RL, corroborating the findings of \citet{irpan2019OPC} and \citet{paine2020hyperparameter}. 

\begin{figure}[h]
    \centering
    \includegraphics[valign=b,scale=0.58,trim={0.5cm 0 0 0}]{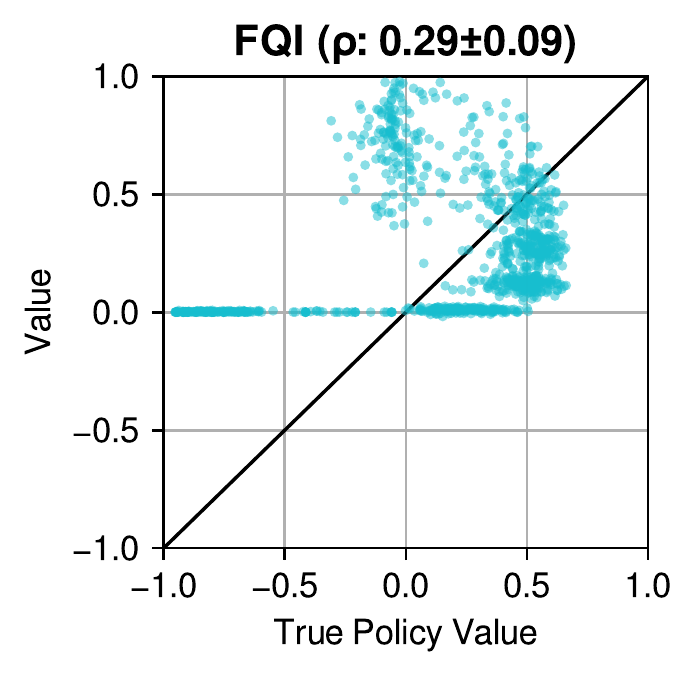}
    \includegraphics[valign=b,scale=0.58]{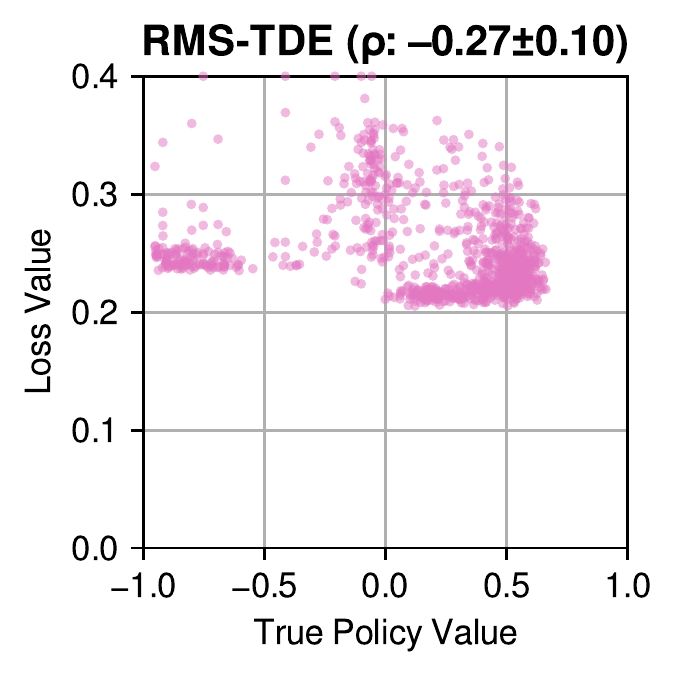}
    \includegraphics[valign=b,scale=0.72,trim={0 -0.3cm 0.5cm 0}]{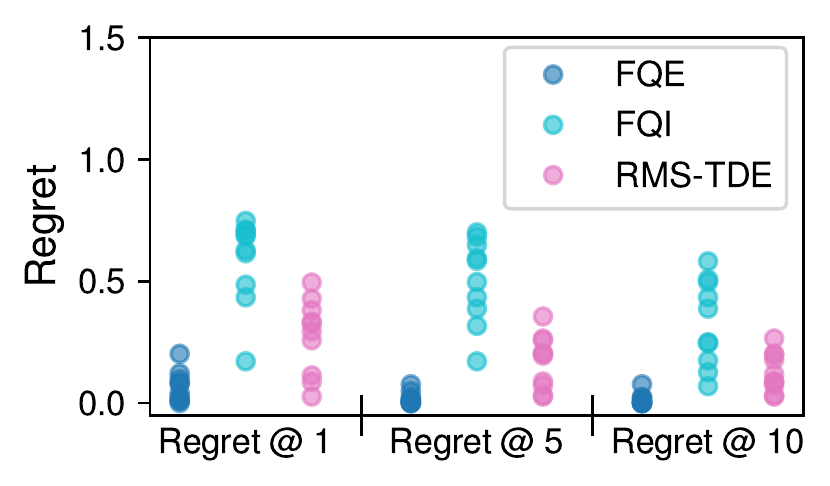}
    \caption{Left -- Scatter plots comparing true policy value with FQI values and RMS-TDE. All values are clipped to the range of possible returns $[-1,1]$ for clarity of visualization. Right -- top-$n$ regret in model selection for FQI and RMS-TDE.}
    \label{fig:sepsis-exp-HP-appx}
\end{figure}

\subsection{Sensitivity Analyses for Early Stopping} \label{appx:early-stopping}
In this section, we focus on the model selection problem of early stopping, i.e., identifying the best number of training iterations in FQI. This allows us to consider both tabular and function approximation settings. The candidate policy set is made up of all the polices derived from Q-functions along the training path of an FQI run. We focus our experiments on WIS/AM/FQE, because WDR resulted in large variances and heavily skews the plots. 

\subsubsection{Sensitivity to OPE Hyperparameters} \label{appx:results-HP}
Since each OPE method has its own auxiliary hyperparameters, it is important to understand how they may affect OPE validation ranking quality, especially as there is no effective way of selecting them. In this experiment, we focus on the hyperparameters that must be set \textit{a priori}, namely: the policy softening parameter $\varepsilon$ and the evaluation horizon $H$ for AM/FQE. We vary these hyperparameters in the reasonable ranges and measure the corresponding regret following model selection (similar to \cref{sec:results-HP}). For continuous state space models, even though the auxiliary models contain hyperparameters themselves (e.g., the neural architecture), these can be tuned via standard model selection within each auxiliary supervised learning task; for consistency, we kept these hyperparameters the same throughout. Note that in the tabular setting, AM can be computed analytically (instead of via Monte-Carlo rollouts) and does not have any hyperparameters, and is therefore omitted from this experiment. 

\begin{figure}[h]
    \centering
    \begin{tabular}{cc}
        \textsf{\textbf{Sepsis (discrete states)\phantom{x}}} & \textsf{\textbf{Sepsis (continuous states)\phantom{xx}}} \\
        \includegraphics[scale=0.58,trim={0.5cm 0 0 0}]{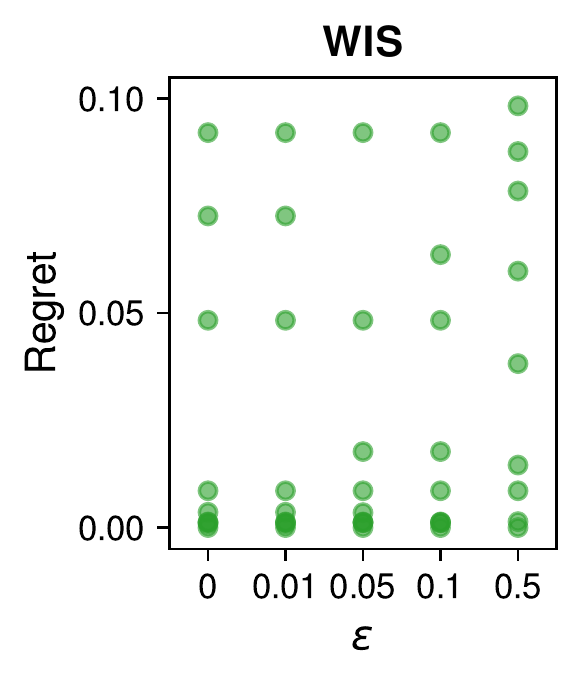} 
        \includegraphics[scale=0.58,trim={1.7cm 0 0 0},clip]{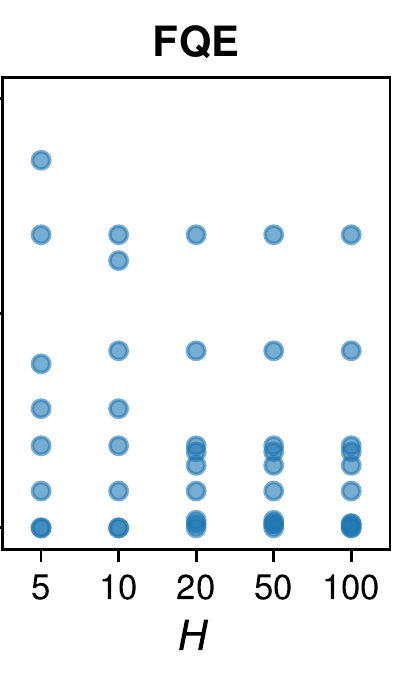} \hspace{0.75cm} &\hspace{0.5cm} \includegraphics[scale=0.58,trim={2.5cm 0 0 0}]{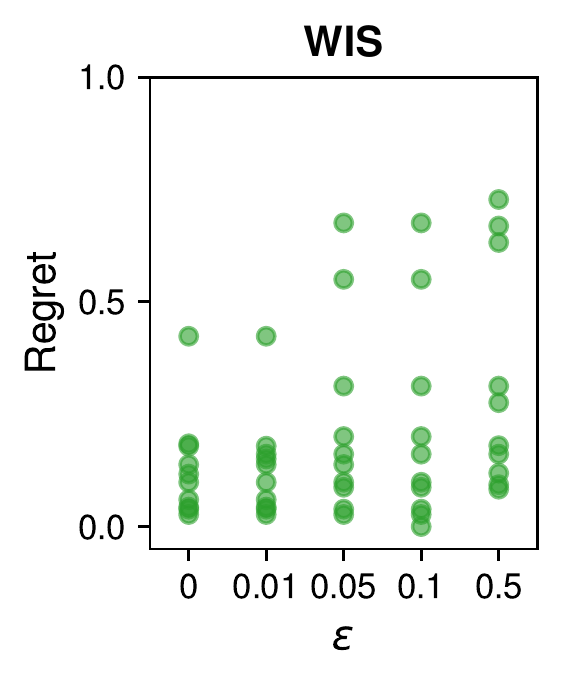}  \includegraphics[scale=0.58,trim={1.5cm 0 0 0},clip]{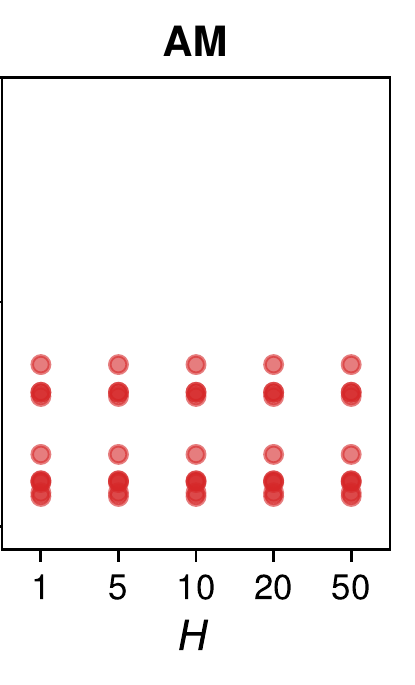} \includegraphics[scale=0.58,trim={1.5cm 0 0 0},clip]{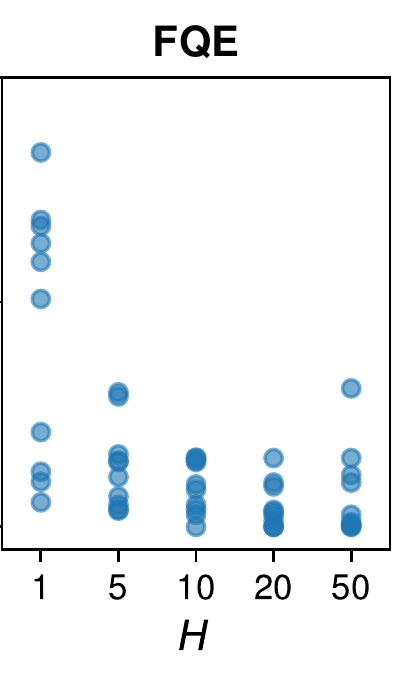} 
    \end{tabular}
    \caption{Regret comparisons for model selection with different OPE hyperparameters. }
    \label{fig:sepsis-exp4-appx}
\end{figure}

\paragraph{Results and Discussion.} For the tabular setting (\cref{fig:sepsis-exp4-appx} left), both WIS and FQE are relatively robust to hyperparameter values ($\varepsilon$ and $H$, respectively) within the reasonable ranges, maintaining a low regret of $<0.1$. For the continuous state setting (\cref{fig:sepsis-exp4-appx} right), although AM appears the most robust to different $H$ values, the regret it obtains is higher than the other two OPE and does not decrease as $H$ increases. WIS obtains the lowest regret at a low $\varepsilon$, where using $\varepsilon$ values that are too large ($\tilde{\pi}$ is too far from $\pi$) led to regret with slightly larger variance. FQE achieves the lowest regret at an intermediate value of $H=20$. This could be due to a form of overfitting within auxiliary models that could cause longer evaluation horizons to have poorer estimates when approximation errors are present \citep{jiang2015dependence}. 

% \begin{figure}[h]
%     \centering
%     \includegraphics[width=.7\linewidth]{figs/plot-legend.pdf}
%     \includegraphics[scale=0.8]{figs/sepsis-tab-exp2-size.pdf} \quad \includegraphics[scale=0.8]{figs/sepsis-cont-exp2-size.pdf}
%     \caption{OPE on varying validation sizes}
%     \label{fig:sepsis-exp2}
% \end{figure}

\subsubsection{Size of Validation Set} \label{appx:results-size}
In this experiment, we vary the amount of validation data used for OPE, similar to \cref{sec:results-size}. The training set and training procedures were kept the same to generate the same candidate set for each run, so the only difference is the validation sample size. % For sepsis simulator, we used validation sets containing the following number of episodes: $\{500, 1000, 5000, 10000\}$.

\begin{figure}[h]
    \centering
    \begin{tabular}{c}
        \quad \textsf{\textbf{Sepsis (discrete states)}} \\
        \phantom{\includegraphics[valign=t,scale=.8]{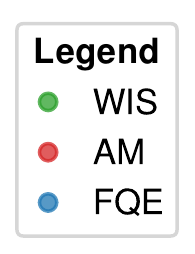}} 
        \hspace{-0.5cm}
        \includegraphics[valign=t,scale=0.75,trim={0 0 0 0.675cm},clip]{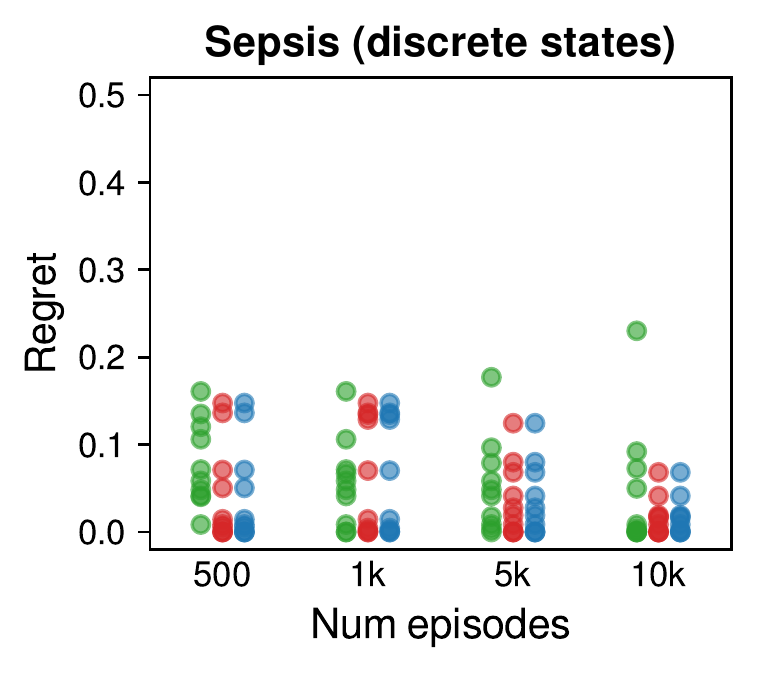} \quad \includegraphics[valign=t,scale=0.75,trim={0 0 0 0.675cm},clip]{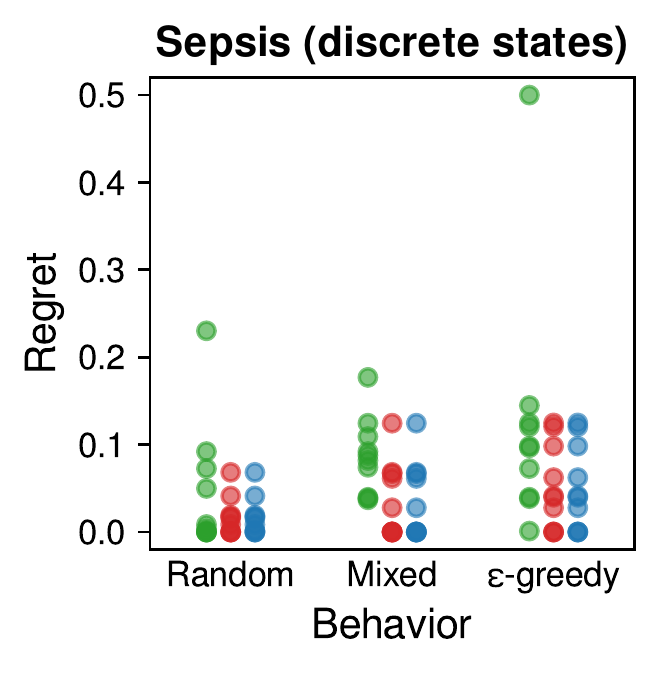}  \includegraphics[valign=t,scale=.8]{figs/plot-legend-appx.pdf} \\
        \quad \textsf{\textbf{Sepsis (continuous states)}} \\
        \hspace{-0.5cm}
        \includegraphics[scale=0.75,trim={0 0 0 0.675cm},clip]{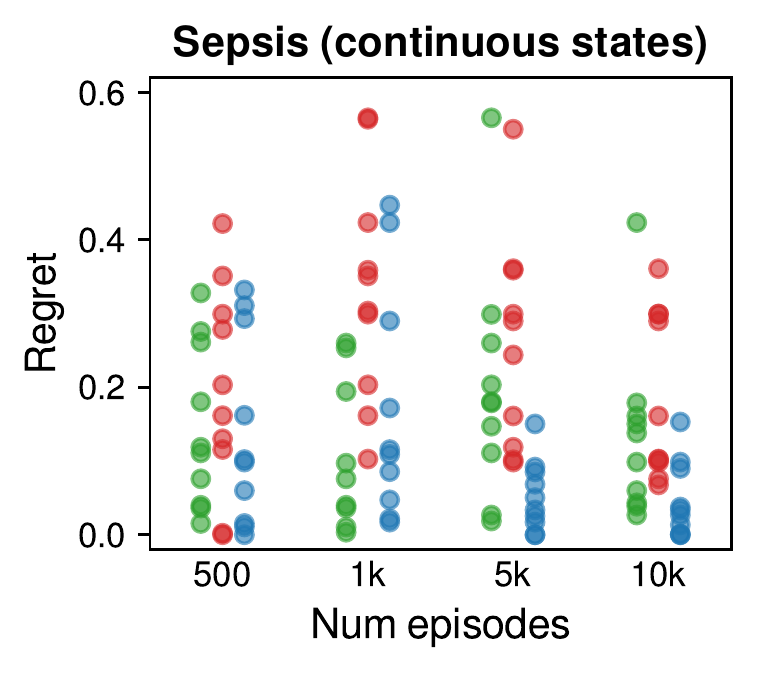} \quad \includegraphics[scale=0.75,trim={0 0 0 0.675cm},clip]{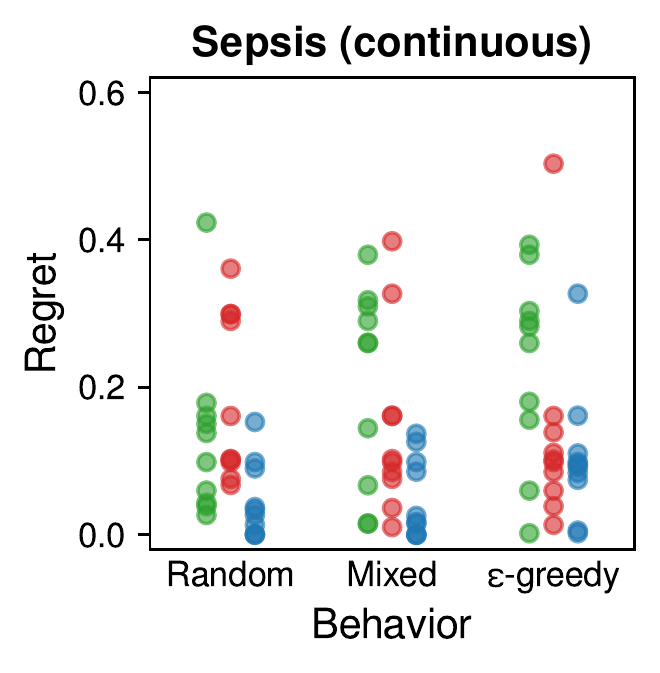}
    \end{tabular}
    \caption{Regret comparison for OPE using validation data with varying sample sizes and varying behaviors. A larger validation set and more exploratory behavior generally results in the lowest regret. }
    \label{fig:sepsis-exp2-exp3-appx}
\end{figure}

\paragraph{Results and Discussion.} For settings with both discrete and continuous states  (\cref{fig:sepsis-exp2-exp3-appx}, left column), FQE consistently benefits from larger validation sizes, achieving lower regret. The trend is less apparent for WIS (which generally results in noisier policy value estimates) and for AM (due to model misspecification, especially with function approximation). 

% Another important consideration is the size of validation set relative to the training set (10k in all cases). While in supervised learning it is typical the validation set is much smaller than the training set, in our offline RL setting, using a large validation set (e.g., as large as the training set) can help alleviate the negative effects of overfitting to the greatest extent. %This effect is also present although slightly weaker for WDR methods. 

\subsubsection{Type of Behavior Policy and Extent of Exploration} \label{appx:results-behavior}
The behavior policies used to collect validation data can affect which part of the state-action space receives sufficient exploration and the resultant OPE model selection performance. Similar to \cref{sec:results-behavior}, we consider the random behavior policy, the near-optimal $\epsilon$-greedy behavior policy, and a mixture of both to create the validation sets each containing $m=10,000$ episodes. The training set and training procedures were kept the same to generate the same candidate set for each run, so that the only difference comes from different validation data. 

\paragraph{Results and Discussion.}
The random behavior policy produces more exploratory data compared to $\epsilon$-greedy, leading to slightly lower regret in general (\cref{fig:sepsis-exp2-exp3-appx}, right column). Surprisingly, even when a mixture of multiple behavior policies is used, WIS is not heavily affected by this policy misspecification (though in the tabular setting, the mixed behavior data produces slightly higher median regret). FQE and AM remain relatively stable across different behavior settings and lead to the lowest regret in most cases with smaller variance.

\section{Additional Related Work} \label{appx:related-work}
The quality of policies learned by RL methods can depend heavily on hyperparameter choices, especially when combined with deep learning \citep{henderson2018matters}. In supervised domains, model selection and hyperparameter tuning has been an essential tool to improve model performance in numerous application settings \citep{bergstra2012random,li2017hyperband,snoek2012practical}. Yet model selection for offline RL has received relatively less attention, especially in empirical analyses on real-world datasets. Closely related to our setup, \citet{paine2020hyperparameter} demonstrated that FQE is a useful strategy for hyperparameter selection in offline RL on several benchmark control tasks; \citet{fu2021benchmarks} published a benchmark suite for comparing OPE methods on continuous control tasks with Spearman's correlation as one of evaluation metrics, and also found FQE to be the most effective; \citet{mandel2014educational} combined importance sampling with a ``temporal'' cross-validation technique and applied their approach on an educational game problem. In contrast to these works, we compare multiple OPE estimators both qualitatively and quantitatively to address questions of practical importance for healthcare settings. There are several other notable works in this area, for which we provide a brief overview below. \citet{fard2010PAC} established PAC-Baysian bounds for offline RL model selection but assumed explicit prior information about environment dynamics or value functions. \citet{farahmand2011BErMin} proposed \textsc{BErMin}, which focuses exclusively on Bellman error minimization rather than policy value maximization; however, due to overestimation of value functions during training, these errors/losses are not reflective of the true policy performance (corroborated by our results in \cref{appx:HP-scatter} and by \citet{irpan2019OPC,paine2020hyperparameter} who showed the poor ranking quality of the loss function). \citet{irpan2019OPC} developed off-policy classification to rank policies, but their approach is limited to goal-directed domains with a sparse binary reward. \citet{lee2020BOPAH} proposed a gradient-based hyperparameter optimization procedure but is limited to a single hyperparameter that controls the extent of KL-regularization. \citet{kuzborskij2021confident} considered a model selection pipeline similar to ours for contextual bandits and developed a new scoring function, whereas we focus on the sequential setting in offline RL and practical challenges in applying existing OPE methods as scoring functions. \citet{xie2020BVFT} studies the problem from a theoretical perspective and proposed BVFT, reducing model selection to pairwise comparisons between candidate policies and provided finite-sample guarantees; in contrast, we consider the more practical aspects of how model selection can (and should) be carried out when applying existing RL/OPE approaches on observational dataset. 

\newpage
\section{Additional Details of the Two-Stage Selection Procedure} \label{appx:two-stage}
In this appendix, we offer some theoretical insight for the expected behavior of the two-stage selection procedure. We first present a general result on the effect of random ranking. 

\begin{theorem} \label{thm:prob_regret}
Suppose we have a list of $C$ values $[v_{[1]}, \cdots, v_{[C]}]$ listed in decreasing order. Given a random permutation $[v_{1}, \cdots, v_{C}]$ over these $C$ values, the probability that the first $A$ positions contain at least one of the $B$ largest values is:
\begin{align*}
    \quad \Pr(\max\{v_1,\cdots,v_{A}\} \geq v_{[B]}) = 
    \begin{cases}
    1 & \text{if } A+B > C \\
    1 - \frac{(C-A)!(C-B)!}{C!(C-A-B)!} & \text{if } A+B \leq C
    \end{cases}
\end{align*}
\end{theorem}

\begin{proof}
We consider the following two cases: 
\begin{itemize}
    \item If $A+B > C$, then the $B$ largest values do not fit in the $C-A$ positions because $B > C-A$, and thus the desired probability is $1$. 
    \item If $A+B \leq C$, then we can consider the complement, i.e., the probability such that none of the $B$ largest values are in the first $A$ positions. This can be obtained by counting the possibilities of arranging $B$ objects into $C$ positions (for the sample space) and $C-A$ positions (for the event space, such that none of the $B$ objects fall into the first $A$ positions), given by \(\frac{(C-A) \cdots (C-A-B+1)}{C \cdots (C-B+1)} = \frac{(C-A)!(C-B)!}{C!(C-A-B)!}\). The desired probability is thus \(1 - \frac{(C-A)!(C-B)!}{C!(C-A-B)!}\). \qedhere
\end{itemize}
\end{proof}

\begin{remark}
Below we show example values of $A$, $B$, $C$, and the corresponding probability.
\begin{table}[h]
    \centering
    \scalebox{0.8}{
    \begin{tabular}{cccc}
        \toprule
        $C$ & $A$ & $B$ & Probability \\
        \midrule
        $1$ & $1$ & $1$ & $1$ \\
        $K$ & $\alpha$ & $1$ & $\alpha/K$ \\
        $K$ & $1$ & $\beta$ & $\beta/K$ \\
        \bottomrule
    \end{tabular}
    }
\end{table}
\end{remark}

\begin{remark}
We can view this probability as a cdf, with the corresponding pdf being the probability that the maximum value of the first $A$ positions being exactly the $B$-th largest value $v_{[B]}$ (assuming all values are distinct, wlog). When $A=1$, we can show that the pdf is $1/C$, regardless of the value of $B$. 
\begin{align*}
    &\quad \Pr(v_1 = v_{[B]})  \\
    &= \Pr(v_1 \geq v_{[B]}) - \Pr(v_1 \geq v_{[B-1]}) \\
    &= \textstyle \left( 1 - \frac{(C-1)!(C-(B+1))!}{C!(C-1-(B+1))!} \right) - \left( 1 - \frac{(C-1)!(C-B)!}{C!(C-1-B)!} \right) \\
    &= \textstyle \frac{(C-1)!(C-B+1)!}{C!(C-B)!} - \frac{(C-1)!(C-B)!}{C!(C-B-1))!} \\
    &= \textstyle \frac{C-B+1}{C} - \frac{C-B}{C} \\
    &= 1/C
\end{align*}

\noindent This matches with our intuition because a randomly selected value from a list of $C$ values should have $1/C$ probability of being any one of the $C$ values. While this seems trivial, the general formula we derived in \cref{thm:prob_regret} allows us to compute the probability for any values of $A$ and $B$. 
\end{remark}

\subsection{Performance of Two-Stage Selection}

Suppose that a random ranking over candidate policies is used in the first-stage. We can use \cref{thm:prob_regret} to calculate the probability that a good policy is retained in the pruned subset, by setting $C=K$ the number of candidate policies, and $A=\alpha$ the initial subset size. Here, $B$ controls the definition of a ``good'' policy -- where setting $B=\beta$ means a ``good'' policy is one of the top $\beta$ policies according to the ground-truth policy values -- and this corresponds to the maximum allowable regret. 

In \cref{fig:2stage-analysis} (left), we plot this probability against initial subset size $\alpha$, for different $\beta$ values. At $\beta=1$, the probability of achieving zero regret increases linearly as the initial subset size $\alpha$ grows. However, when a larger level of regret is allowed (larger $\beta$ values), the probability is high even for small $\alpha$ values. This suggests that, for example, pruning the candidate set by half has a high chance of retaining a good policy even if the first-stage OPE has random performance. Since OPE estimators are expected to perform better than random, the above probability provides the lower bound; a (trivial) upper bound can be obtained from an optimal ranking that has probability $1$ for any values of $A$ and $B$, including $A=B=1$. 

\begin{figure}[h]
    \centering 
    \includegraphics[valign=c,scale=0.75]{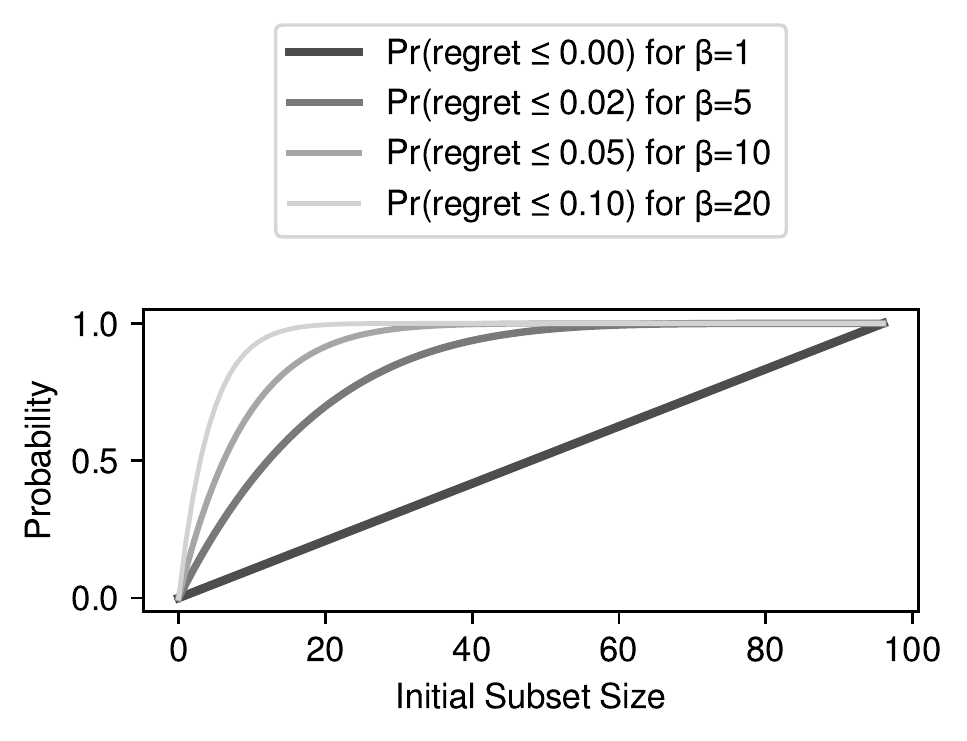}
    \includegraphics[valign=c,scale=0.75]{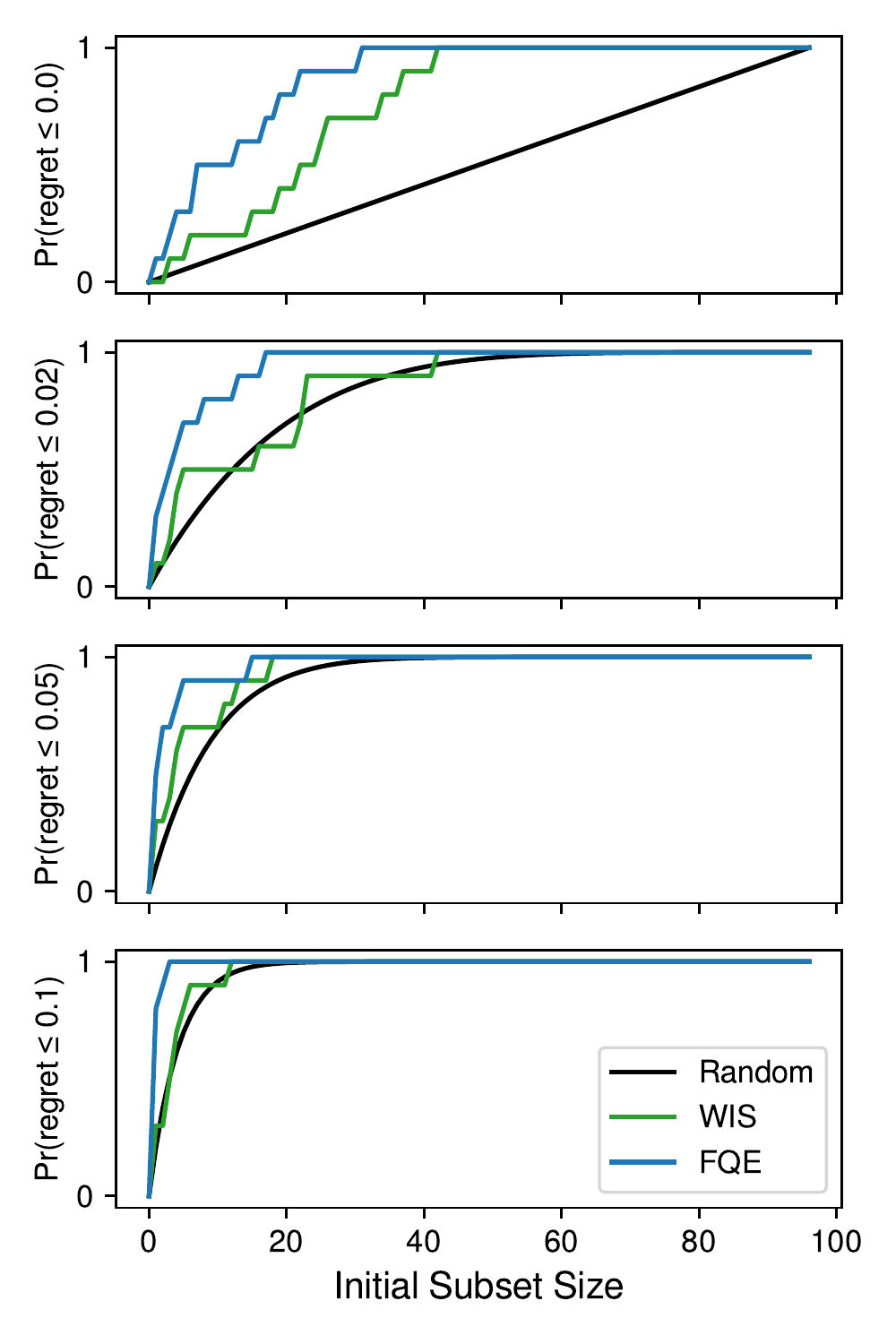}
    \caption{Comparison of cumulative distribution functions of a random OPE, WIS, and FQE, relative to initial subset size $\alpha$, for the 96 candidate policies on the sepsis simulator task for different $\beta$ values (which controls the level of allowable regret). }
    \label{fig:2stage-analysis}
\end{figure}

To compare the performance of WIS and FQE against the random lower bound, we computed empirical cdfs from the 10 replication runs for each OPE method, using the main experiment in 
\cref{sec:results-arch}. In \cref{fig:2stage-analysis} (right), we visualize these cdfs on a set of axes, one for each $\beta$ value. We see that in general, FQE has a higher chance of retaining a ``good'' policy than WIS, if used for the first-stage pruning, for all $\alpha$ and $\beta$ values. Nonetheless, both WIS and FQE have cdfs larger than the random lower bound, suggesting that both OPE estimators indeed perform better than random. 

\subsection{Further Analyses \& Discussion}
The analyses above focused on the sole effect of the first-stage pruning, and thus assumed that the second OPE in the two-stage procedure produces a perfect ranking. One could estimate the overall performance by taking the product of probabilities of the two stages, assuming that rankings from the two stages are independent. However, the independence assumption does not hold in reality, because all OPE validation scores are generated from the same validation data. Furthermore, the quality of ranking from an OPE is challenging to quantify as it depends on many factors: the OPE itself, the data, the variance of each value estimate (this can differ per candidate policy), etc. These analyses are outside the scope of this paper and left for future work. 

\end{document}